\documentclass[lettersize,journal]{IEEEtran}
\usepackage{amsmath,amsfonts}
\usepackage{algorithmic}
\usepackage{algorithm}
\usepackage{array}
\usepackage[caption=false,font=normalsize,labelfont=sf,textfont=sf]{subfig}
\usepackage{textcomp}
\usepackage{stfloats}
\usepackage{url}
\usepackage{verbatim}
\usepackage{graphicx}
\usepackage{cite}

\usepackage{microtype}
\usepackage{amsthm}
\usepackage{mathtools}
\usepackage{bm}
\usepackage{booktabs}
\usepackage[colorlinks,citecolor=blue,linkcolor=blue,urlcolor=black]{hyperref}
\usepackage[capitalize]{cleveref}
\usepackage{subfig}
\usepackage{tikz}
\usetikzlibrary{bayesnet}
\usepackage{tabularx}
\usepackage{multirow}
\usepackage{multicol}
\usepackage{adjustbox}
\usepackage{csquotes}
\usepackage{microtype}
\usepackage{xcolor} 
\usepackage{colortbl} 
\usepackage{orcidlink}

\usepackage{todonotes}

\definecolor{mayablue}{rgb}{0.45, 0.76, 0.98}

\newcommand{\E}{\mathbb{E}}
\newcommand{\m}{\mathbf}
\newcommand{\x}{\mathbf{x}}
\newcommand{\y}{\mathbf{y}}
\newcommand{\z}{\mathbf{z}}
\newcommand{\yhat}{\hat{\mathbf{y}}}
\newcommand{\zhat}{\hat{\mathbf{z}}}
\newcommand{\ytilde}{\tilde{\mathbf{y}}}
\newcommand{\kl}{D_{KL}}
\newcommand{\method}{LSNPC}
\newcommand{\vocseven}{VOC07}
\newcommand{\voctwelve}{VOC12}
\newcommand{\coco}{COCO}
\newcommand{\tomato}{Tomato}
\newcommand{\pair}{\textsc{Pair}}
\newcommand{\sym}{\textsc{Sym}}
\newcommand{\macrof}{macro-F1}
\newcommand{\microf}{micro-F1} 

\newtheorem{assumption}{Assumption}
\newtheorem{corollary}{Corollary}
\newtheorem{theorem}{Theorem}
\newtheorem{lemma}{Lemma}
\newtheorem{proposition}{Proposition}

\crefname{assumption}{Assumption}{Assumptions}

\crefname{eq}{Eq.}{Eqs.}
\Crefname{eq}{Equality}{Equalities}

\begin{document}

\title{Correcting Noisy Multilabel Predictions:\\ Modeling Label Noise through Latent Space Shifts}

\author{
Weipeng Huang\IEEEauthorrefmark{1}\orcidlink{0000-0003-4620-6912}, 
Qin Li\IEEEauthorrefmark{1}\orcidlink{0000-0003-0902-9371},  
Yang Xiao\orcidlink{0000-0003-1410-0486}, 
Cheng Qiao\orcidlink{0000-0003-2887-5549}, 
Tie Cai,  
Junwei Liang\IEEEauthorrefmark{2}\orcidlink{0000-0003-1999-0254}, 
Neil J. Hurley\orcidlink{0000-0001-8428-2866},
Guangyuan Piao\orcidlink{0000-0003-0516-2802} 

\IEEEcompsocitemizethanks{
\IEEEcompsocthanksitem \IEEEauthorrefmark{1}  
Equal contributions.
\IEEEcompsocthanksitem \IEEEauthorrefmark{2} 
Corresponding author.
}

\IEEEcompsocitemizethanks{
\IEEEcompsocthanksitem 
W. Huang, Q. Li, T. Cai, and J. Liang are with the School of Computer Science and Software Engineering, Shenzhen Institute of Information Technology, Guangdong, China. Email: \{weipenghuang, liqin, cait, jwliang\}@sziit.edu.cn 
\IEEEcompsocthanksitem 
Y. Xiao is with the State Key Lab of Integrated Service Networks, School of Cyber Engineering, Xidian University, Shaanxi, China. Email: yxiao@xidian.edu.cn
\IEEEcompsocthanksitem
C. Qiao is with the Cyberspace Institute of Advanced Technology, Guangzhou University, Guangdong, China. Email: mcheng.qiao@gmail.com
\IEEEcompsocthanksitem N. J. Hurley is with the School of Computer Science, University College Dublin, Dublin, Ireland. Email: neil.hurley@ucd.ie
\IEEEcompsocthanksitem G. Piao is an independent researcher. Email: parklize@gmail.com

\IEEEcompsocthanksitem The code is available at 
\url{https://github.com/huangweipeng7/lsnpc}. 
}

}

\markboth{Journal of \LaTeX\ Class Files,~Vol.~14, No.~8, August~2021}%
{Shell \MakeLowercase{\textit{et al.}}: A Sample Article Using IEEEtran.cls for IEEE Journals}


\maketitle

\begin{abstract}
Noise in data appears to be inevitable in most real-world machine learning applications and would cause severe overfitting problems.
Not only can data features contain noise, but labels are also prone to be noisy due to human input.
In this paper, rather than noisy label learning in multiclass classifications, we instead focus on the less explored area of noisy label learning for multilabel classifications.
Specifically, we investigate the post-correction of predictions generated from classifiers learned with noisy labels.
The reasons are two-fold.
Firstly, this approach can directly work with the trained models to save computational resources.
Secondly, it could be applied on top of other noisy label correction techniques to achieve further improvements.
To handle this problem, we appeal to deep generative approaches that are possible for uncertainty estimation.
Our model posits that label noise arises from a stochastic shift in the latent variable, providing a more robust and beneficial means for noisy learning.
We develop both unsupervised and semi-supervised learning methods for our model.
The extensive empirical study presents solid evidence to that our approach is able to consistently improve the independent models and performs better than a number of  existing methods  across various noisy label settings.
Moreover, a comprehensive empirical analysis of the proposed method is carried out to validate its robustness, including sensitivity analysis and an ablation study, among other elements.
\end{abstract}


\section{Introduction}
\IEEEPARstart{M}{ultilabel} classification is a task where a single data point can be associated with multiple labels~\cite{Bogatinovski2022comp}.
In contrast, multiclass classification assigns only one single label to each data point, and is thought of as a simpler task~\cite{liu2017easy}. 
For modern deep learning problems, in particular for multilabel classifications, model performance is hugely impacted by the data and its labeling quality. 
However, label noise can be generated through various means and is often inevitable when creating large datasets, particularly due to human errors~\cite{welinder2010multidimensional,xiao2015learning,cheng2020learning,han2018co,liu2020early,wang2020relaxed}.
For instance, inconsistencies may occur because multiple labeling experts, working independently, have differing subjective interpretations of the labeling guidelines.
On the other hand, implementing a shared labeling process that relies on majority voting can prove to be costly for large and complex data.
It is therefore crucial to tackle the challenges posed by noisy labels in multilabel classification tasks.

So far, the research of noisy label learning has mostly focused on the problem of multiclass classification, including~\cite{welinder2010multidimensional,xiao2015learning,cheng2020learning,han2018co,liu2020early,wang2020relaxed,Bae2022from,lu2023selc,Zhong2023Neighbour,jiang2024leveraging,yang2022estimating,li2021provably,tu2023learning}, to name but a few. 
As highlighted in the survey~\cite{Song2023learning}, label noise in multilabel classification tasks is more difficult to handle for two main reasons: 
1) learning the label correlation from noisy labels is non-trivial to settle; 
2) more importantly however, the label sparsity/imbalance leads to a more challenging situation.
Consequently, research of noisy label learning in multilabel classification has been less active.
In particular, Xia et al.~\cite{xia2023holistic}, enhancing the method of~\cite{Li2022est}, propose a statistical score to leverage the label correlation for noisy label correction during the training phase, before the model overfits prohibitively.
More recently, Chen et al.~\cite{chen2024unm} propose a unified framework incorporating semantic embedding and label embedding to design label noise resilient models.
To the best of our knowledge, there is no prior work exploring how predictions from models built with noisy multilabels can be corrected.

Our goal is to develop a general framework which is capable of correcting the predictions made by a classifier trained with noisy labels, for minimizing the gap to the true labels. 
Bae et al.~\cite{Bae2022from} first propose the technical path of calibrating the noisy predictions (namely post-processing). Their model---noisy prediction calibration (NPC)---copes with noisy labels for multiclass classifications. 
This stream of methods preserves two advantages: 1) we can apply it to the pre-trained models without easy access to re-training; 2) we can employ it on top of other noisy label handling techniques to acquire further improvements. 
In NPC~\cite{Bae2022from}, a deep generative approach which adopted and a variational auto-encoder (VAE)~\cite{Kingma2014vae} was proposed to infer the generative model.
The success of this work heavily relies on the property of multiclass classification and thus cannot be trivially extended to the case of multilabel classification.
We will fully characterize this disconnection in~\cref{sec:ext}, and present empirical evidence to support our statements in~\cref{sec:ext_exp}. 

This work adheres to the deep generative modeling procedure that first establishes a parameterized generative process involving the use of deep neural networks.
Incorporating uncertainty by considering a range of probability distributions proves to be effective in combating the impact of noise.
Thereafter, we estimate the corresponding parameters through learning from the  observed data.
In a typical VAE, one usually defines a latent variable $\z$ which generates the random variable matching the observations.
One may regard this latent variable as the deep factors for reconstructing the observations.
Alternatively, it can be thought of as the clustered features of the observations~\cite{Kingma2014vae,kingma2014semi}.
In our scenario, $\z$ is responsible for generating the variable of (noisy) labels. 
The heuristics of our work is clear: the noisy labels are actually generated due to the shift in latent space such that $\z$ turns to the shifted one $\zhat$.  
Despite being rarely investigated~\cite{engel2018latent}, latent variable shift is a very intuitive yet valid approach, which perfectly suits our scenario. 
The advantage of using this approach will reasonably mitigate the impact of label correspondence and imbalance, as we indeed attempt to capture the underlying factors of the label noise generation.
Once the shifted latent variable still locates in the right latent space, the generated label noise will also follow the pattern (in particular for the label correspondence and imbalance) of the true labels. 
With careful design of the generative process, we are able to develop a learning process that helps correct the noisy predictions to get closer to the true predictions.
Additionally, given the nature of the model, it leaves room for us to also develop a Bayesian-flavored semi-supervised learning paradigm, following the methodology outlined in~\cite{kingma2014semi}.
Unlike the setting in~\cite{xia2023holistic}, we posit there exists a small set of clean data for validation, which is affordable by most companies or organizations and thus a pragmatic assumption.
Moreover, we theoretically unveil the connection between our model objective function and the goal of learning the true latent variables.

Our contributions are highlighted here.
Firstly, we conduct a thorough analysis to understand why extending the problem from the multiclass to the multilabel scenario is not a trivial task for NPC.
Secondly, we propose a novel deep generative model for post-processing the noisy predictions in multilabel classification.
Since our \underline{n}oisy \underline{p}rediction \underline{c}orrection focuses on the \underline{l}atent variable \underline{s}hift, we thereafter name it \method{}.
This approach roots from the deep Bayesian methodology which is also a less explored branch of methods in handling noisy label corrections. 
In addition, we theoretically analyze the properties of the model which further justify our model design.
Finally, extensive experiments demonstrate that our model achieves significantly solid improvements for all the examined datasets and noise settings.  


\section{Related Works}
\label{sec:related}
Noisy labels are inevitable in realworld datasets, posing a significant challenge to the generalization ability of deep neural networks trained on them. These deep neural networks can even easily fit randomly assigned labels in the training data, as demonstrated in~\cite{zhang2021understanding}. To tackle this challenge, many studies have been proposed for learning with noisy labels. For example, noise-cleansing methods primarily aim to separate clean data pairs from the corrupted dataset using the output of the noisy classifier~\cite{malach2017decoupling,wang2021proselflc,kim2021fine,han2018masking,tanaka2018joint,han2018co,yu2019does,wei2020combating,zheng2021meta,zheng2020error}. 
Another line of research focuses on designing either explicit regularization or robust loss functions for training to mitigate the problem of noisy labels~\cite{liu2020early,xia2020robust,zhang2018generalized,wang2019symmetric,ma2020normalized}. 
To explicitly model  noise patterns, another branch of research suggests using a transition matrix $T$ to formulate the noisy transition from true labels to noisy labels~\cite{patrini2017making,yao2020dual,zhang2021learning,xia2020part,berthon2021confidence}. 
Different from these lines of research, Bae et. al~\cite{Bae2022from} introduced NPC (Noisy Prediction Calibration), which is a new branch of method working as a post-processing scheme that corrects the noisy prediction from a pre-trained classifier to the true label.

Despite these endeavors of tackling learning with noisy labels, a recent survey~\cite{song2022learning} points out that the majority of the existing methods are applicable only for a
single-label multiclass classification problem, and more research is required for the multilabel classification problem where each example can be associated with multiple true class labels. 
Our approach focuses on multilabel classification along with recent studies.

More recently, several works have been proposed for multilabel classification with noisy labels.
For example, HLC~\cite{xia2023holistic} uses instance-label and label dependencies in an example for follow-up label correction during training. 
UNM~\cite{chen2024unm} uses a label-wise embedding network that semantically aligns label-wise features and label embeddings in a joint space and learns the co-occurrence of multilabels. The label-wise embedding network cyclically adjusts the fitting status to distinguish the clean labels and noisy labels, and generate pseudo labels for the noisy ones for training. 
Our approach is orthogonal to these studies, and can be seen as a post-processor similar to NPC~\cite{Bae2022from} for multiclass classification. 
In other words, our method can be used to correct the predictions of a pre-trained classifier such as HLC, and further improve the performance as a post-processor as we show in~\cref{sec:experiments}.

\section{Problem Setup} 
\label{sec:problem}
This section formally describes the technical problem.
Let us consider a multilabel classification task containing $k$ labels.
Let $\x \in \mathcal{X} \subseteq \mathbb{R}^d$ be a $d$-dimensional data point and $\y \in \mathcal{Y} = \{0, 1\}^k$ be the \emph{true} corresponding label vector where each element is binary---$1$ indicates that the data point contains the corresponding label, whereas $0$ indicates the opposite.
The dataset $\mathcal{D}$ is defined by $\mathcal{D} = \{(\x_i, \y_i)\}_{i=1}^n$.
Given our task for handling the noisy labels, the observable labels $\tilde{\y}$ might be shifted from the true label $\y$.
We thus denote the observed dataset as $\tilde{\mathcal{D}}$ where $\tilde{\mathcal{D}} = \{(\x_i, \tilde{\y}_i)\}_{i=1}^n$.
Finally, we denote the Kullback-Leibler divergence (KL-divergence) by $\kl$.
 
The focus of this paper is the prediction correction for a pre-trained classifier.
The main reason is that this stream of approaches does not conflict with the most other existing methods which try to denoise during the training phase.
In practice, even with the denoising training techniques, the trained classifiers are probably still biased and can benefit from our approach.
We now consider an arbitrary classifier $h: \mathcal{X} \mapsto \mathcal{Y}$ which is trained using the observed data $\tilde{\mathcal{D}}$.
We write the prediction as $\yhat = h(\x)$.
If the model $h(\cdot)$ is stochastic, $\yhat \sim h(\x)$ can be interpreted as the predictive posterior probability
\begin{align*}
& \yhat \sim p(\yhat \mid \x, h)\\
& p(\yhat \mid \x, h)\equiv p(\yhat \mid \x, \tilde{\mathcal{D}}) \eqqcolon p_h(\yhat \mid \x) \,,
\end{align*}
as $h(\cdot)$ is trained on $\tilde{\mathcal{D}}$. 
Compared with the raw noisy labels $\tilde{\y}$, $p_h(\yhat \mid \x)$ can be regarded as an \emph{approximated distribution} over the noisy labels and thus enable the Bayesian modeling.
The task is to learn a label transition model $\mathcal{C}(\cdot)$ for a noisy model $h(\cdot)$ which takes the data as input and outputs the corrected labels. 
Finally, we emphasize the notation in~\cref{tb:notation} for further clarification and reference. 
\begin{table}[ht]
\centering
\caption{Notation Clarification for the Label Variables}\label{tb:notation}
\begin{tabularx}{1\columnwidth}{lX}
\toprule
$\y$ & true labels \\
$\ytilde$ & labels annotated by experts which potentially contain noise \\
$\yhat$ & labels that are predicted through a model $h(\cdot)$ which is trained on the noisy set $\tilde{\cal D}$ (observed labels) \\
\midrule
\bottomrule
\end{tabularx} 
\end{table}

\section{Extension of NPC}
\label{sec:ext}
NPC~\cite{Bae2022from} focuses on multiclass classification and thus we denote the label variable by $y$.
Considering the reconstructed label for $\x$ is $y_{recon}$, the VAE loss is defined as 
\begin{align}
  \mathcal{L}_{recon}(y_{recon}, \hat{y}) + \kl[ q(y_{recon} \mid \x, \hat{y}) \mid\mid p_h(\hat{y} \mid \x) ] \,.
\end{align}
Here, $\mathcal{L}_{recon}$ is the reconstructed loss which is a cross entropy for the categorical distribution (the primary choice of distribution for handling multiclass classification).
The second term is the model regularization where the authors assume a prior of generating $\hat{y}$ by the noisy classifier.

NPC posits that, given an input $\x$, the latent variable for reconstructing the noisy label $\hat{y}$ is the true label $y$.
However, the connection between the latent variable, i.e., the true label $y$, and the noisy label $\hat{y}$ is constructed through a parameterized neural network.
This lacks a directional connection between the latent variable and the noisy label and therefore may lead to an enormous space for parameter search.
In a traditional Bayesian modeling process, one may design and combine the best existing statistical models to describe the transformation from $\hat{y}$ to $y$.
Also, most of the classical Bayesian models have only a few parameters to learn, and their contraction rates and asymptotics~\cite{walker2004modern} are well proven, which makes them more predictable and directed models. 
From an optimization perspective, the success of NPC thus lies in the regularization term which restricts the decoder $q(y_{recon} \mid \x, \hat{y})$ to be close to the noisy classifier $p_h(\hat{y} \mid \x)$, in which $\hat{y}$ serves as the ground truth.
Nevertheless, the NPC allows the model to search for a distribution closer to the classifier's prediction distribution while a reasonable amount of uncertainty allows the model to search for better parameters for the distribution.

According to the rationale, the extension to a multilabel case will fail since the regularization term will be greatly impacted by the sparsity of the labels.
The only choice for the multilabel variables $\y$ and $\yhat$ will be the multivariate Bernoulli distribution.
For two discrete multivariate Bernoulli distributions $P$ and $Q$, the KL-divergence is 
\begin{align*}
\kl[ P \mid\mid Q] = \sum_i p_i \log \frac{p_i}{q_i} 
+ (1-p_i) \log \frac {1-p_i}{1-q_i}  \,.
\end{align*}
Imagine that there are e.g. 20 labels, while, in each instance, there are only 1-3 labels assigned the value of 1.
If two multivariate Bernoulli both set very high probabilities for the labels that are assigned 0, the KL-divergence between them is still small as most of the labels are matched and thus the overall difference is amortized. 
Then the regularization term loses its ability to guide the parameter search in the space. 

To further support our analysis, we implemented this extension and demonstrated in~\cref{sec:experiments} that it does not focus in the right learning directions but wanders stochastically instead.
Our implementation followed~\cite{wang2020relaxed} to employ a Normal copula~\cite{nelsen2006introduction} to ensure that the samples from the multivariate Bernoulli distributions capture the label correlation.

\section{NPC Focusing on Latent Variable Shift (LSNPC)}
\label{sec:model}

Our approach regards the generation of label noise as a random shift in the latent space. 
Assume that a latent variable $\z$ (along with the observation $\x$) is a main factor involved in generating the true label, then $\zhat$ is a randomly shifted variable that plays a role in generating the observed labels (which are potentially noisy). 
In line with the Bayesian modeling procedure, we first detail its generative process and subsequently analyze the posteriors of the model for different learning paradigms.

\subsection{Generative Process}
 
\begin{figure}[t]
  \centering
  \subfloat[Unsupervised]{
  \begin{tikzpicture}
    \node[latent]   (y) {$\y$};
    \node[latent, left=of y, xshift=-0.5cm] (z) {$\z$};
    \node[latent, below=of z] (zhat) {$\zhat$};
    \node[obs, below=of z, xshift=.8cm, yshift=.8cm]  (x) {$\x$};
    \node[obs, below=of y]  (yhat) {$\yhat$};
  
    \edge {x,z} {y} ; %
    \edge {z} {zhat} ; %
    \edge {x, zhat} {yhat} ; %
  \end{tikzpicture}
  }
  \qquad \qquad
  \subfloat[Supervised]{
    \begin{tikzpicture}
    \node[obs]                               (y) {$\y$};
    \node[latent, left=of y, xshift=-0.5cm] (z) {$\z$};
    \node[latent, below=of z] (zhat) {$\zhat$};
    \node[obs, below=of z, xshift=.8cm, yshift=.8cm]  (x) {$\x$};
    \node[obs, below=of y]  (yhat) {$\yhat$};
  
    \edge {x,z} {y} ; %
    \edge {z} {zhat} ; %
    \edge {x, zhat} {yhat} ; %
  \end{tikzpicture}
  }
  \caption{The process of the noisily labelled data generation. The gray background indicate that the variable is observed. \label{fg:generative}}
\end{figure}
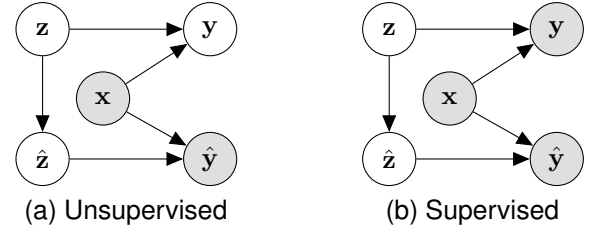
\cref{fg:generative} exhibits the graphical model of the generative process.
The left graph is for the unsupervised setting while the right graph is for the supervised setting.
The only difference is whether the true labels $\y$ is observed.
Next, the latent variable $\z \in \mathcal{Z} \subseteq \mathbb{R}^m$ is the variable for generating the true labels $\y$, while $\zhat \in \mathcal{Z}$ is the shifted latent variable, which combines with $\x$ to generate the noisy labels $\yhat$.
In our scenario, $\yhat$ is a sample drawn from the pre-trained classifier $h(\x)$.
This modeling strategy follows that of the Bayesian linear regression~\cite{Castillo2015bayesian}, which omits the generation of the observation $\x$ but treats it as an input covariate, because this is not the component of interest to our task.
Following the strict Bayesian convention to reconstruct $\x$ is unnecessary and might raise the difficulty of learning of true labels.
The rationale is that its role in the objective function could potentially distract the learning process from focusing on learning the true labels.

We depict the specifications in~\cref{fg:generative} before elucidating all the details:
\begin{align*}
\z &\sim \mathrm{Normal}(\m{0}, \m{I}) \\
\zhat \mid \z &\sim \mathrm{Student}(g_{\psi}(\z), \m{I}, \nu_0) \\ 
\y \mid \z, \x &\sim \mathrm{Bernoulli}(g_{\phi}(\x, \z)) \quad \\
\yhat \mid \zhat, \x &\sim \mathrm{Bernoulli}(g_{\phi}(\x, \zhat)) \,.
\end{align*}
The latent variable $\z$ is a typical multivariate Normal distribution with zero mean and identity covariance matrix.
Then, $\zhat$ is a random variable of the potentially \enquote{shifted neighbor} of $\z$.
We design $\zhat$ to be drawn from a multivariate Student distribution with a hyperparameter $\nu_0$.
The heavy tailed property of the distribution helps the whole process become more robust and in particular effective against noise~\cite{GAO2021on,li2021t,Li2020a}.  
The mean of this Student distribution is transformed through a decoder $g_{\psi}(\z)$ where $\psi$ is the corresponding parameter set. 
This step models the process of shifting $\z$.
Also, Student VAEs~\cite{Takahashi2018student,Abiri2019vae} have been studied and the corresponding reparameterization tricks are available.

Other than the latent variables, $\y$ is the random variable for true labels.
We define $\yhat$ as the random variable representing the observed labels from the data which can be either clean or noisy. 
As defined, the true latent variable $\z$ is responsible for constructing the true labels $\y$, along with $\x$.
On the contrary, the observed labels $\yhat$ are then constructed based on $\x$ and $\zhat$.
Therefore, $\yhat$ can be numerically close to $\y$; and in this case, the latent $\zhat$ should also be numerically close to the true latent variable $\z$.
We emphasize that the generation of $\y$ and $\yhat$ share the same multivariate Bernoulli distribution.
Apart from that, they also share same the decoder function $g_{\phi}(\cdot)$, parameterized with $\phi$.

It is noteworthy that $\zhat$ should locate in the same space as $\z$ since it is a neighbor of $\z$. 
Consequently, $\zhat$ operates in the same support as $\z$ and should generate the labels using the same decoder as $\z$, when $\x$ is fixed.
Finally, we define a set for all the parameters in the decoder function by $\Phi_d = \{\psi, \phi \}$ for simplicity.

\subsection{Variational Auto-Encoder}

In this subsection, we detail the proposed distributions of this generative model, crucial for learning the generative models in VAEs.
We begin with the unsupervised learning setting, followed by the supervised setting.
Combining the two parts, we can derive the semi-supervised learning solution.

\subsubsection{Unsupervised Learning}
For the unsupervised fashion where $\y$ is unobserved, we consider the following marginal probability $p(\x, \yhat)$.
Let us denote the corresponding evidence lower bound (ELBO) by 
\[
\mathrm{ELBO} = \E_{\z, \zhat \sim q(\z, \zhat \mid \x, \yhat)}\left[\frac{p(\x, \yhat, \z, \zhat)}{q(\z, \zhat \mid \x, \yhat)} \right] \,.
\]
Hence, we show
\begin{align}
\label{eq:unsup_obj}
\log p(\x, \yhat)
&\ge \E_{\z, \zhat \sim q(\z, \zhat \mid \x, \yhat)}\left[\frac{p(\x, \yhat, \z, \zhat)}{q(\z, \zhat \mid \x, \yhat)} \right] \notag \\
&= \E_{\zhat \sim q(\zhat \mid \x, \yhat)} [\log p(\yhat \mid \x, \zhat)] \notag \\
&\quad + \E_{\z, \zhat \sim q(\z, \zhat \mid \x,\yhat)} [\log p(\zhat \mid \z) + \log p(\z)] \notag \\
&\quad - \E_{\zhat \sim q(\zhat \mid \x, \yhat)}[ \log q(\zhat \mid \x, \yhat)] \notag \\
&\quad - \E_{\z, \zhat \sim q(\z, \zhat \mid \x,\yhat)}[\log q(\z \mid \zhat)] \,
\end{align}
where  
\begin{align}
q(\z, \zhat \mid \x, \yhat) = q(\z \mid \zhat) q(\zhat \mid \x, \yhat) \,.
\end{align}
Furthermore, we specify
\begin{align} 
q(\z \mid \zhat)
&= \mathrm{Normal}(\z; \mu_{\kappa}(\zhat), \mathrm{diag}(\sigma^2_{\kappa}(\zhat))) \label{eq:unsup_zhat0} \\
q(\zhat \mid \x, \yhat)
&= \mathrm{Student}(\zhat; \mu_{\theta}(\x, \yhat), \mathrm{diag}(\sigma^2_{\theta}(\x, \yhat)), \nu) \label{eq:unsup_zhat}  
\end{align}
where $\mu_{\kappa}: \mathcal{Z} \mapsto \mathcal{Z}$ and $\sigma^2_{\kappa}:\mathcal{Z} \mapsto \mathcal{Z}$ are the corresponding encoder functions for $\zhat$ with respect to mean and covariance matrix; likewise, $\mu_{\theta}: \mathcal{X} \times \mathcal{Y} \mapsto \mathcal{Z}$ and $\sigma^2_{\theta}: \mathcal{X} \times \mathcal{Y} \mapsto \mathcal{Z}$ are the encoder functions for $\x, \yhat$.
We denote the parameters for the encoder function by $\Phi_e =\{ \kappa, \theta \}$.

In the variational Bayes framework, the objective is always to maximize the ELBO~\cite{Kingma2014vae,kingma2014semi}.
It is clear that maximizing the ELBO is equivalent to minimizing the following KL-divergence:
\begin{align} 
&\max \E_{\z, \zhat \sim q(\z, \zhat \mid \x, \yhat)}\left[\frac{p(\x, \yhat, \z, \zhat)}{q(\z, \zhat \mid \x, \yhat)} \right] \notag \\
&= \min \kl[q(\z, \zhat \mid \x, \yhat) \mid\mid p(\z, \zhat \mid \x, \yhat)] - \log p(\yhat, \x) \notag \\
&\equiv \min \kl[q(\z, \zhat \mid \x, \yhat) \mid\mid p(\z, \zhat \mid \x, \yhat)] \label{eq:obj_relation} \,.
\end{align} 
In the proposed distributions, we set $\nu$ as a hyperparameter rather than learning it from a parameterized encoder function, for two main reasons: 1) setting a fixed value suffices to perform well (check the empirical study section); 2) we can set a value which fits in our theoretical study. 

Given~\cref{eq:unsup_obj,eq:obj_relation}, we can summarize the loss function for the unsupervised learning as 
\begin{align}
  \label{eq:unsup_obj_summarized}
  \mathcal{U}(\tilde{\mathcal{D}}) 
  \coloneqq \sum_{(\x, \_) \in \tilde{\mathcal{D}}, \yhat \sim h(\x) }\kl [q(\z, \zhat \mid \x, \yhat) \mid\mid p(\z, \zhat \mid \x, \yhat)] \, ,
\end{align}
where $\tilde{\mathcal{D}}$ is the training set that contains the noisy labels.
 
\paragraph{Correction Phase}
We can now specify the label correction process given that \method{} has been explained.
Let us denote the corrected labels for $\x$ by $\y^*$.
Following the rhythm of Bayesian learning, we define the label correction function $\mathcal{C}(\cdot)$ as 
\begin{align} 
\label{eq:correction}
\y^*
&= \mathcal{C}(\x; h) = \E[\y \mid \x, h] \nonumber \\
&= \int \int \y \cdot p(\y \mid \x, \yhat) p_h(\yhat \mid \x) d \y d \yhat \nonumber \\
&\approx \int \int \y \left[ \int \int p(\y \mid \x, \z) q(\z \mid \zhat) q(\zhat \mid \x, \yhat) d\z d\zhat \right] \nonumber \\
&\quad \cdot p_h(\yhat \mid \x) d \y d \yhat \nonumber \\
&= \E_{\y \sim p(\y \mid \x, \z),\z \sim q(\z \mid \zhat), \zhat \sim q(\zhat \mid \x, \yhat), \yhat \sim p_h(\yhat \mid \x)} [\y ] \, . 
\end{align}
In practice, we approximate this equation by Monte Carlo sampling.
This correction function $\mathcal{C}(\cdot)$ remains consistent across unsupervised, supervised, and semi-supervised paradigms.

\subsubsection{Supervised Learning}
For the supervised fashion where $\y$ is observed, we consider the marginal probability which follows
\begin{align}
\label{eq:sup_obj}
&\log p(\x, \y, \yhat) \notag \\
&\ge \E_{\z, \zhat \sim q(\z, \zhat \mid \x, \y, \yhat)}\left[\frac{p(\x, \y, \yhat, \z, \zhat)}{q(\z, \zhat \mid \x, \yhat)} \right] \notag \\
&= \E_{\z, \zhat \sim q(\z, \zhat \mid \x, \y, \yhat)} [\log p(\yhat \mid \x, \zhat) + \log p(\y \mid \x, \z) \notag \\
&\quad + \log p(\zhat \mid \z) + \log p(\z) ] - \E_{\zhat \sim q(\zhat \mid \x, \yhat, \zhat)}[ \log q(\zhat \mid \x, \yhat)] \notag \\
&\quad - \E_{\z \sim q(\z \mid \x, \y)} [\log q(\z \mid \x, \y)] \,,
\end{align} 
where 
\begin{align} 
q(\z, \zhat \mid \x, \y, \yhat)
= q(\z \mid \x, \y, \zhat) q(\zhat \mid \x, \yhat) \,.
\end{align}
We further propose
\begin{align}
q(\z \mid \x, \y, \zhat)
&= \eta \mathrm{Normal}(\z; \mu_{\theta}(\x, \y), \mathrm{diag}(\sigma^2_{\theta}(\x, \y))) \notag \\
&\quad + (1-\eta) q(\z \mid \zhat) \label{eq:sup_z} \\
q(\zhat \mid \x, \yhat)
&= \mathrm{Student}(\zhat; \mu_{\theta}(\x, \yhat), \mathrm{diag}(\sigma^2_{\theta}(\x, \yhat)), \nu) \,. \label{eq:sup_zhat} 
\end{align}
Apparently,~\cref{eq:sup_zhat} is identical to the proposed distribution in the unsupervised setting (\cref{eq:unsup_zhat}).
However, $q(\z \mid \x, \y, \zhat)$ is designed to be a mixture model of two weighted distributions, controlled by $\eta$.
On one hand, we are learning the distribution $q(\z \mid \zhat)$ as that in the unsupervised setting. 
This distribution is crucial because, during the correction phase, this distribution will be the sole component to rely on, since the true label $\y$ is unknown. 
On the other hand, heuristically, clean data could be used to enhance the encoder function parameters contained in $\mathrm{Normal}(\z; \mu_{\theta}(\x, \y), \mathrm{diag}(\sigma^2_{\theta}(\x, \y)))$. 
This step improves the learning of the encoder functions that can only be learned through the noisy data in the unsupervised configuration.
During our empirical exploration, we found that values of $\eta$ within $(0, 1)$ result in nearly identical performance.
The reason maybe that, provided the two distributions are drawn a sufficient number of times during learning, the parameters are able to visit the region of the search space in which the noisy predictions can be corrected to their best. 
The number of training epochs remained the same and showed negligible influence.
However, setting $\eta$ to exactly $0$ or $1$ would lead to slightly worse performance.
With $\eta=0$, we lose the opportunity to apply the clean data to correct the learning of the decoder functions which are mainly trained by the noisy data in the unsupervised part.
In contrast, setting $\eta=1$ disconnects $q(\z \mid \zhat)$, the main component for label correction in the inference phase (see~\cref{eq:correction}), from the clean data during training.  

Similarly, denoting $\mathcal{D}$ as the clean dataset for training, we define the objective function as 
\begin{align}
\mathcal{L}(\mathcal{D}) 
\coloneqq \sum_{(\x, \y) \in \mathcal{D}} \kl [q(\z, \zhat \mid \x, \y, \yhat) \mid\mid p(\z, \zhat \mid \x, \y, \yhat)] \,.
\end{align}

\subsubsection{Semi-Supervised Learning}
Following the paradigm in the work~\cite{kingma2014semi}, we define the corresponding loss function as 
\begin{align}
\min_{\bm{\Phi}} \mathcal{L}(\mathcal{D}) + \mathcal{U}(\tilde{\mathcal{D}}) \, 
\end{align}
where $\bm{\Phi} = \Phi_d \cup \Phi_e$.
However, we split the training into two consecutive training parts which is depicted in~\cref{alg:train}.
\begin{algorithm}[tp]
\caption{Semi-Supervised Training} \label{alg:train}
\begin{algorithmic}[1]
\REQUIRE{$\xi$, the learning rate}
\FOR{every epoch}
  \STATE \hfill\COMMENT{Unsupervized learning with $\tilde{\mathcal{D}}$}
  \FORALL{ batch $\tilde{\mathcal{B}}$ in the noisy data $\tilde{\mathcal{D}}$}
    \STATE $\bm{\Phi} \gets \bm{\Phi}- \xi \cdot \nabla_{\bm{\Phi}} \mathcal{U}(\tilde{\cal{B}})$ 
  \ENDFOR
  \STATE \hfill\COMMENT{Semi-Supervised learning if $\mathcal{D}$ is available}
  \FORALL{batch $\mathcal{B}$ in the clean data $\mathcal{D}$}
    \STATE $\bm{\Phi} \gets \bm{\Phi} - \xi \cdot \nabla_{\bm{\Phi}}\mathcal{L}(\cal{B})$ 
  \ENDFOR
\ENDFOR
\end{algorithmic}
\end{algorithm} 
This suffices to leverage the clean data for correcting the learning.

\subsection{Theoretical Properties}
In this section, we analyze the theoretical properties of our proposed method.
Due to the space limit, we defer the proofs to the supplemental material.

In our loss function under the unsupervised setting, when considering a single data point $\x$, the (expected) KL-divergence between the proposal $q(\z \mid \zhat)$ and the marginalized posterior $p(\z \mid \x, \yhat)$ is also being minimized.
\begin{theorem} \label{lem:upper}
Minimizing the objective in~\cref{eq:unsup_obj_summarized} is equivalent to minimizing the upper bound of the expected KL-divergence between the marginalized true posterior on $\z$ and the proposed marginalized one.
That is, for a given data point $\x$, we obtain
\begin{multline}
  \E_{\zhat \sim q(\zhat \mid \x, \yhat)}[\kl[q(\z \mid \zhat) \mid\mid p(\z \mid \x, \yhat) ]] \\
  \le \kl[ q(\z, \zhat \mid \x, \yhat) \mid\mid p(\z, \zhat \mid \x, \yhat)] \,.
\end{multline}
\end{theorem}
Therefore, the distribution $q(\z \mid \zhat)$ learned through \method{} is guaranteed to move close to the true marginal posterior $p(\z \mid \x, \yhat)$, facilitating the approximation of the true posterior.

Moreover, we analyze the gap between the marginalized approximate posterior $q(\z \mid \x, \yhat)$ under the condition that the observed labels $\yhat$ are identical to the ground-truth labels and the scenario in which $\yhat$ corresponds to other arbitrary label values.
It is crucial to investigate the impact of $\yhat$ on learning $\z$, as $\yhat$ directly influences only the learning of $\zhat$.
Suppose the clean ground-truth values of $\x$ are denoted by $\y_0$ and a vector of other label values is denoted by $\y_1$.
In VAEs, the distribution $q(\z \mid \x, \yhat=\y_0)$ could well differ from $q(\z \mid \x, \yhat=\y_1)$. 
We are interested in a scenario where the noisy labels should be able to help learn the true latent variable $\z$ or a close neighbor, and thus the distribution distance of $q(\z \mid \x, \yhat=\y_0)$ and $q(\z \mid \x, \yhat=\y_1)$ should be bounded. 
The distance between these distributions associates with the robustness of the selected proposal distributions using a noisy set of labels along with $\x$ to infer the true latent variable $\z$.
It is expected that the distance between such two distributions is constrained.

Let $\mathcal{S}_{\kappa}$ and $\mathcal{S}_{\theta}$  represent the space for the parameters, $\kappa$ and $\theta$ respectively. 
Furthermore, let $\Delta(\cdot, \cdot)$ be a symmetric metric in binary $k$-dimensional space $\mathcal{Y}$.
The value will be greater than 1 if two label values being compared are not identical; otherwise, the value will be 0.
We posit the following assumptions.
During any phase in the training process, the following conditions hold.
\begin{assumption} \label{assmp:sigma_L_lip}
We assume that for the encoder function $\sigma_{\theta}(\cdot, \cdot)$, given that $\forall \x \in \mathcal{X}, \theta \in \mathcal{S}_{\theta}, \y_0, \y_1 \in \mathcal{Y}$, $\exists M: 0 < M < \infty$ such that,  
\begin{align}
  \max_{j=1, \dots, m} \frac{\sigma_{\theta}^2(\x, \y_0)_{j}}{\sigma_{\theta}^2(\x, \y_1)_{j}} \le M \Delta(\y_0, \y_1) \,.
\end{align}
\end{assumption} 

\begin{assumption} \label{assmp:mu_L_lip}
We posit the Lipschitz continuity for the encoder function $\mu_{\kappa}$ over the label space $\mathcal{Y}$. 
Given that $\forall \x \in \mathcal{X}, \kappa \in \mathcal{S}_{\kappa}, \y_0, \y_1 \in \mathcal{Y}, \exists L: 0 < L < \infty$,  
\begin{align} 
\| \mu_{\kappa}(\x, \y_0) - \mu_{\kappa}(\x, \y_1) \|_2
&\le |\mu_{\kappa}(\x, \y_0) - \mu_{\kappa}(\x, \y_1)| \notag \\
&\le L \Delta(\y_0, \y_1) \,.
\end{align}
\end{assumption} 
The first inequality is a trivial result of norm inequalities while the second one assumes a typical Lipschitz continuity for optimization functions~\cite{nesterov2018lectures}.
\begin{assumption} 
\label{assmp:min_mat_val}
$\exists \lambda > 0$, 
\[
\lambda \le \inf~\{\min \sigma_{\theta}^2(\x, \y): \x \in \mathcal{X}, \y \in \mathcal{Y}, \theta \in \mathcal{S}_{\theta} \} \,.
\]
\end{assumption}
It is also a practical assumption as we should not allow zero variance for any dimension in a multivariate Student distributed dimension.  
Otherwise, the learning would not succeed. 
 
\begin{theorem}
\label{thm:z_bound}
Let $m$ denote the dimension of the latent variable $\z$ and $\zhat$.
Also, let $\Psi(\cdot)$ be the digamma function.
Given an observation $\x$ and its true label value $\y_0$, the corresponding true latent variable is $\z$. 
Imagine a vector of noisy label value $\y_1$ and $\Delta(\y_0, \y_1) < \infty$. 
When $\nu = \nu_0 > 2$, the following inequality holds.
\begin{multline}
\kl[ q(\z \mid \x, \yhat = \y_1) \mid\mid q(\z \mid \x, \yhat = \y_0)] \\
\le C_1 + C_2 \Delta(\y_0, \y_1) \label{eq:gap_ineq} \,  
\end{multline}
where, letting $\alpha = \frac{\sqrt{\nu \lambda}}{L}$,  
\begin{align*} 
  C_1 &=  \frac{\nu+m}{2} \left\{ \frac{M \sqrt{m} \alpha }{2(\nu-2) L} - \Psi\left( \frac{\nu + m}{2} \right) + \Psi\left(\frac{\nu}{2} \right) \right\} \\
  C_2 &= \frac{m M}{2e} + \frac{(\nu+m)\sqrt{m}}{2\alpha} \,.
\end{align*}
\end{theorem}
It immediately follows that the KL-divergence is bounded by $O(m \sqrt{m} \Delta(\y_0, \y_1))$, while if the proposals are replaced by multivariate Normal distributions, the divergence is then bounded by $O(m \Delta(\y_0, \y_1)^2)$\footnote{The theoretical analysis is placed in Section C of the supplemental material.}.
This provides a perspective that the Student distribution could be a more noise-resistant distribution when noise is large.
To emphsize, these are the worst case theoretical bounds but the actual differences could vary depending on the specific cases.
Practically, Normal distributions could perform better than the Student distributions in some cases, as discussed in~\cref{sub:as} of our experiments. 
Nevertheless, we argue that an upper bound with a close-to-zero coefficient on the distance metric is not necessarily better. 
Such a bound might not assist the model in terms of empirical performance, since it is even more detrimental to allow any noisy labels to learn the true latent variables.



\section{Experiments}
\label{sec:experiments}

In this section, we first introduce experimental settings such as datasets, baselines for comparison, and evaluation metrics\footnote{To maintain the flow of the main text, the implementation details are presented in Section D of the supplemental material.}. 
Next, we present the empirical analysis of the direct extension of NPC
Then, we demonstrate the results for both the unsupervised and semi-supervised \method{}.
Finally, we conduct the sensitivity analysis and ablation study respectively.

\subsection{Experimental Design}
\label{sub:exp_design}
\paragraph{Datasets} To evaluate the performance of our proposed method---\method{}, we used four benchmark datasets: \vocseven{}~\cite{everingham2007pascal}, \voctwelve{}~\cite{everingham2007pascal}, \coco{}~\cite{lin2014microsoft}, and \tomato{}~\cite{gehlot2023tomato}. 
\cref{tb:dataset} summarizes the statistics of each dataset. 
\vocseven{} and \voctwelve{} contain images from 20 common object labels, and \coco{} contains images from 80 common object labels. For both the \vocseven{} and \voctwelve{} datasets, we used the \vocseven{} test set which contains 4,952 images as in~\cite{xia2023holistic}. 
The \tomato{} dataset contains images from 8 tomato diseases.  
 
We generated synthetic noisy labels following designs from previous studies~\cite{xia2023holistic, patrini2017making, han2018co} with a class-dependent noise transition matrix $T$, where $T_{ij} = p(\tilde{y}=j \mid y=i)$ refers to the probability of the $i$-th label to be flipped into the $j$-th label. 
More specifically, we adopt two noise settings: symmetric (\sym)~\cite{van2015learning} and pairflip (\pair)~\cite{han2018co}, where \sym{} noise flips labels uniformly to other labels and \pair{} noisifies the specific labels with the similar ones. 
Following~\cite{xia2023holistic}, for each noise setting, we take into account various levels of noise (specifically, $0\%, 30\%, 40\%, 50\%$), which is referred to as the noise rate (NR), to verify the effectiveness of \method{} calibration under different noise conditions. 
In addition, we include the setting for clean data to analyze its generalization ability.
The clean subset of data consists of a randomly selected half of the validation set.

\begin{table}[t]
	\centering
	\caption{Dataset statistics including the number of examples in training, validation, clean and test set as well as the number of labels.}
	\begin{tabular}{c c c c c c}
		\toprule
		\textbf{Dataset} & \textbf{Training} & \textbf{Validation} & \textbf{Clean} & \textbf{Test} & \textbf{Labels} \\ \midrule
		\vocseven & 4,509		& 251	& 251 & 4,952 & 20 \\ 
		\voctwelve & 10,386	& 577 & 577 & 4,952  & 20 \\ 
		\coco	& 82,081	& 10,034 & 10,034 & 20,069	& 80 \\ 
		\tomato	& 3,945		& 562	& 562 & 584	& 8	 \\ \bottomrule
	\end{tabular}
	\label{tb:dataset}
\end{table}

\paragraph{Baselines} To demonstrate the performance improvements with \method{} as a post-processor, we apply it to a pre-trained classifier with baseline methods. These  methods include: 
\begin{itemize}
\item 
MLP is a baseline multilabel classifier (enabled by a Sigmoid layer appended to the last position) with a fully connected layer on top of the extracted features from an image encoder. We selected two typical image encoders: 
\begin{itemize}
	\item ResNet-50~\cite{he2016deep} where the entire model is denoted by $\mbox{MLP}_{r}$; 
	\item LeViT~\cite{alexey2020image} where the entire model is denoted by $\mbox{MLP}_{v}$.
\end{itemize}
\item 
ADDGCN~\cite{ye2020attention} uses a semantic attention module to estimate the content-aware
class-label representations for each class from  extracted feature map where these representations are fed into a graph convolutional network (GCN) for final classification.
\item 
HLC~\cite{xia2023holistic} is a noisy multilabel correction approach built on top of ADDGCN. It uses the ratio between the holistic scores of the example with noisy multilabels and its variant with predicted labels to correct noisy labels during training. A holistic score measures the instance-label and label dependencies in an example.
\end{itemize}
To evaluate the robustness of \method{} to the choice of backbone image encoder, we run experiments with ResNet-50~\cite{he2016deep} (a ResNet based encoder) and LeViT~\cite{graham2021levit} (an attention based encoder) as the image encoder of each baseline method.
However, ADDGCN and HLC are tightly integrated with the ResNet-50 model, and therefore the ResNet-50 retained as the image encoder.
For a fairer comparison, the image encoder in \method{} is fixed to be consistent with those in the base models, being either ResNet-50 or LeViT.
In all settings, \method{} applies the same image encoder as the base model to avoid the situation where a better image encoder brings about the improvements rather than the framework itself.

\paragraph{Evaluation Metrics} 
We utilized \microf{} which measures F1 score across all labels, and \macrof{} which measures the average inter-label F1-measure as our evaluation metrics for evaluating multi-label classification performance.
These two metrics are the most practically important metrics in multilabel classification~\cite{zhu2024hill,wang2022contrastive,zhang2022use,zhang2024multi}. 

\begin{figure}[!t]
\centering
\subfloat[]{
\includegraphics[width=0.42\textwidth]{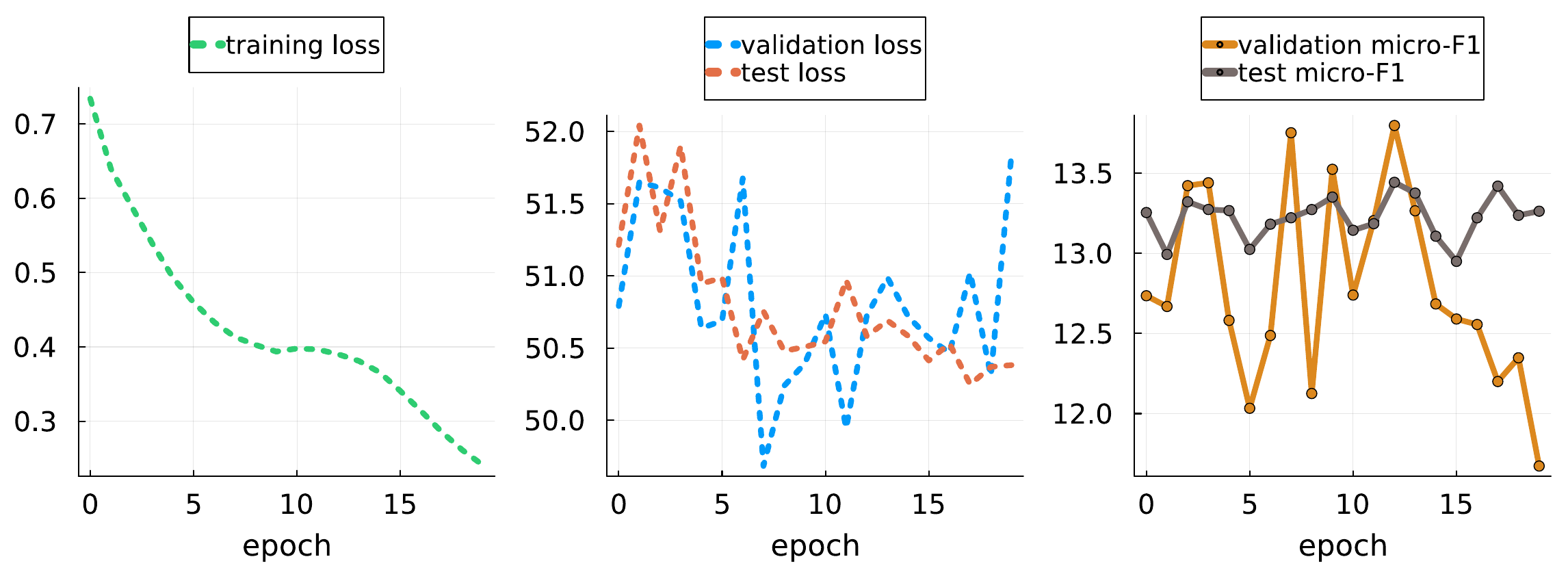}
} \\
\subfloat[]{
\includegraphics[width=0.42\textwidth]{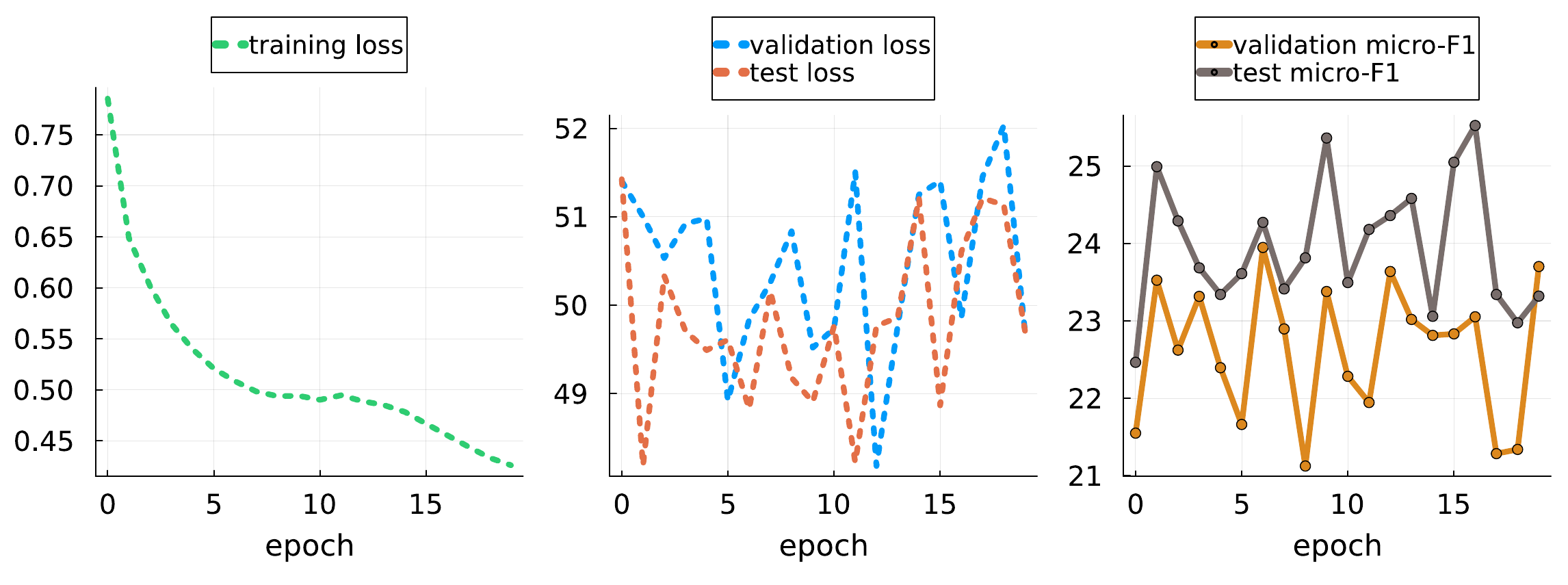}
}
\caption{
Training analysis of extension of NPC.  
(a) Extension of NPC on \vocseven. The base model is chosen to be $\mbox{MLP}_r$. The noise setting is \pair{} with 0.3 NR.
(b) Extension of NPC on \tomato. The base model is chosen to be ADDGCN. The noise setting is \sym{} with 0.5 NR.
} \label{fg:ext_npc}
\end{figure}

\begin{table*}[!t]
  \centering  
  \caption{Unsupervised learning results for \vocseven and \tomato. The standard deviations are reported.} \label{tb:unsup}
  \begin{adjustbox}{width=1\textwidth} 
  \begin{tabular}{cclcccccccc} 
  \toprule
  & & & \multicolumn{4}{c}{\vocseven} & \multicolumn{4}{c}{\tomato} \\
  \cmidrule(l){4-7} \cmidrule(l){8-11} 
  & & & \multicolumn{2}{c}{\sym} & \multicolumn{2}{c}{\pair} & \multicolumn{2}{c}{\sym} & \multicolumn{2}{c}{\pair} \\
  \cmidrule(l){4-5} \cmidrule(l){6-7} \cmidrule(l){8-9} \cmidrule(l){10-11}  
   & NR & post-processor & macro-F1 & micro-F1 & macro-F1 & micro-F1 & macro-F1 & micro-F1 & macro-F1 & micro-F1 \\ 
  \midrule \midrule
  \multicolumn{1}{c}{} & \multirow{3}{*}{0\%} & Baseline & $\mathbf{84.0}$ & $\mathbf{85.7}$ & & & $20.4$ & $47.2$ & & \\ 
  \multicolumn{1}{c}{\multirow{12}{*}{$\mbox{ADDGCN}$}} & & KNN & $81.1$ & $81.8$ & & & $23.6$ & $45.7$ & & \\
  \rowcolor{mayablue!30}
  \cellcolor{white}  & \cellcolor{white}  & \method{} & $83.8 \pm 0.3$ & $85.1 \pm 0.2$ & & & $\mathbf{23.8 \pm 1.7}$ & $\mathbf{49.1 \pm 0.7}$ & & \\ 
  \cmidrule(l){2-11} 
  \multicolumn{1}{c}{} & \multirow{3}{*}{30\%} & Baseline & $74.1$ & $74.5$ & $69.6$ & $71.4$ & $17.9$ & $37.4$ & $15.8$ & $34.7$ \\
  \multicolumn{1}{c}{} & & KNN & $64.9$ & $65.9$ & $50.2$ & $54.8$ & $14.2$ & $31.7$ & $7.9$ & $17.2$ \\
  \rowcolor{mayablue!30}	
  \cellcolor{white}  & \cellcolor{white} & \method{} & $\mathbf{79.1 \pm 0.3}$ & $\mathbf{80.4 \pm 0.2}$ & $\mathbf{71.6 \pm 1.3}$ & $\mathbf{72.6 \pm 0.6}$ & $\mathbf{24.7 \pm 2.0}$ & $\mathbf{46.4 \pm 1.4}$ & $\mathbf{26.5 \pm 2.3}$ & $\mathbf{45.3 \pm 1.0}$ \\
  \cmidrule(l){2-11} 
  \multicolumn{1}{c}{} & \multirow{3}{*}{40\%}    & Baseline & $68.1$ & $68.6$ & $55.1$ & $62.1$ & $14.8$ & $30.4$ & $20.1$ & $36.1$ \\
  \multicolumn{1}{c}{} & & KNN & $57.7$ & $60.8$ & $9.5$ & $13.1$ & $2.8$ & $7.1$ & $12.9$ & $26.4$ \\
  \rowcolor{mayablue!30}
  \cellcolor{white}  & \cellcolor{white} & \method{} & $\mathbf{77.2 \pm 0.4}$ & $\mathbf{78.3 \pm 0.6}$ & $\mathbf{57.3 \pm 2.4}$ & $\mathbf{63.0 \pm 0.7}$ & $\mathbf{24.5 \pm 2.8}$ & $\mathbf{42.5 \pm 0.5}$ & $\mathbf{26.5 \pm 2.0}$ & $\mathbf{43.9 \pm 0.7}$ \\
  \cmidrule(l){2-11} 
  \multicolumn{1}{c}{} & \multirow{3}{*}{50\%} & Baseline & $50.7$ & $59.0$ & $41.4$ & $46.2$ & $12.7$ & $31.0$ & $16.0$ & $27.5$ \\
  \multicolumn{1}{c}{} & & KNN & $6.7$ & $14.3$ & $39.1$ & $42.8$ & $3.2$ & $8.8$ & $6.2$ & $17.6$ \\
  \rowcolor{mayablue!30}
  \cellcolor{white}  & \cellcolor{white} & \method{} & $\mathbf{70.7 \pm 0.6}$ & $\mathbf{76.8 \pm 0.3}$ & $\mathbf{48.8 \pm 0.9}$ & $\mathbf{52.1 \pm 0.8}$ & $\mathbf{20.6 \pm 1.2}$ & $\mathbf{42.2 \pm 0.3}$ & $\mathbf{22.7 \pm 1.5}$ & $\mathbf{35.0 \pm 0.4}$ \\ 
  \midrule \midrule 
   & \multirow{3}{*}{0\%} & Baseline & $\mathbf{84.0} $ & $\mathbf{85.8}$ & & & $10.1$ & $32.2$ & & \\
   \multirow{12}{*}{$\mbox{HLC}$} & & KNN & $81.0$ & $81.9$ & & & $\mathbf{15.4}$ & $\mathbf{44.4}$ & & \\
  \rowcolor{mayablue!30}
  \cellcolor{white}  & \cellcolor{white} & \method{} & $83.7 \pm 0.2$ & $85.2 \pm 0.1$ & & & $12.4 \pm 1.0$ & $36.5 \pm 0.7$ & & \\ 
  \cmidrule(l){2-11} 
  & \multirow{3}{*}{30\%} & Baseline & $78.7$ & $80.2$ & $77.6$ & $80.0$ & $6.5$ & $25.6$ & $7.0$ & $25.9$ \\
  & & KNN & $66.4$ & $68.6$ & $69.0$ & $70.7$ & $5.3$ & $24.4$ & $5.3$ & $24.4$ \\
  \rowcolor{mayablue!30}
  \cellcolor{white}  & \cellcolor{white} & \method{} & $\mathbf{80.4 \pm 0.3}$ & $\mathbf{81.6 \pm 0.2}$ & $\mathbf{79.9 \pm 0.4}$ & $\mathbf{81.5 \pm 0.4}$ & $\mathbf{14.4 \pm 3.5}$ & $\mathbf{29.6 \pm 2.1}$ & $\mathbf{12.8 \pm 2.3}$ & $\mathbf{28.2 \pm 1.2}$ \\
  \cmidrule(l){2-11} 
  & \multirow{3}{*}{40\%} & Baseline & $76.4$ & $77.7$ & $64.4$ & $68.9$ & $5.4$ & $24.4$ & $6.5$ & $25.4$ \\
  & & KNN & $69.2$ & $70.0$ & $51.8$ & $56.5$ & $5.3$ & $24.4$ & $5.3$ & $24.4$ \\
  \rowcolor{mayablue!30}
  \cellcolor{white}  & \cellcolor{white} & \method{} & $\mathbf{79.6 \pm 0.5}$ & $\mathbf{81.2 \pm 0.5}$ & $\mathbf{70.0 \pm 0.9}$ & $\mathbf{71.7 \pm 0.4}$ & $\mathbf{10.4 \pm 3.2}$ & $\mathbf{27.9 \pm 2.5}$ & $\mathbf{13.4 \pm 2.2}$ & $\mathbf{29.1 \pm 1.3}$ \\
  \cmidrule(l){2-11} 
  & \multirow{3}{*}{50\%} & Baseline & $67.9$ & $73.7$ & $36.0$ & $45.1$ & $7.0$ & $34.3$ & $5.4$ & $24.5$ \\
  & & KNN & $53.5$ & $61.1$ & $34.4$ & $44.0$ & $6.8$ & $33.8$ & $5.3$ & $24.4$ \\
  \rowcolor{mayablue!30}
  \cellcolor{white}  & \cellcolor{white} & \method{} & $\mathbf{73.3 \pm 0.4}$ & $\mathbf{78.7 \pm 0.4}$ & $\mathbf{42.3 \pm 0.6}$ & $\mathbf{49.5 \pm 0.4}$ & $\mathbf{9.5 \pm 1.0}$ & $\mathbf{35.4 \pm 1.0}$ & $\mathbf{13.5 \pm 6.2}$ & $\mathbf{26.7 \pm 1.3}$ \\ 
  \midrule \midrule
   & \multirow{3}{*}{0\%} & Baseline & $\mathbf{83.7}$ & $\mathbf{85.6}$ & & & $15.6$ & $39.0$ & & \\
  \multirow{12}{*}{$\mbox{MLP}_r$} & & KNN & $80.5$ & $82.5$ & & & $1.2$ & $3.0$ & & \\
  \rowcolor{mayablue!30}
  \cellcolor{white}  & \cellcolor{white} & \method{} & $83.2 \pm 0.5$ & $84.9 \pm 0.2$ & & & $\mathbf{21.8 \pm 2.5}$ & $\mathbf{43.0 \pm 0.7}$ & & \\
  \cmidrule(l){2-11} 
  & \multirow{3}{*}{30\%} & Baseline & $74.9$ & $76.4$ & $69.7$ & $73.1$ & $11.5$ & $21.3$ & $10.8$ & $21.9$ \\
  & & KNN & $65.5$ & $70.2$ & $47.4$ & $62.4$ & $0.5$ & $0.6$ & $1.7$ & $3.3$ \\
  \rowcolor{mayablue!30}
  \cellcolor{white}  & \cellcolor{white} & \method{} & $\mathbf{78.0 \pm 0.6}$ & $\mathbf{79.7 \pm 0.4}$ & $\mathbf{71.3 \pm 0.9}$ & $\mathbf{74.9 \pm 0.5}$ & $\mathbf{28.2 \pm 0.8}$ & $\mathbf{36.0 \pm 1.3}$ & $\mathbf{27.9 \pm 2.8}$ & $\mathbf{37.8 \pm 2.4}$ \\
  \cmidrule(l){2-11} 
  & \multirow{3}{*}{40\%} & Baseline & $67.3$ & $68.9$ & $57.4$ & $61.4$ & $9.5$ & $17.4$ & $10.9$ & $18.2$ \\ 
  & & KNN & $51.8$ & $56.1$ & $21.4$ & $24.2$ & $0.2$ & $0.3$ & $0.2$ & $0.3$ \\
  \rowcolor{mayablue!30}
  \cellcolor{white}  & \cellcolor{white} & \method{} & $\mathbf{74.2 \pm 0.7}$ & $\mathbf{76.0 \pm 0.7}$ & $\mathbf{61.0 \pm 0.6}$ & $\mathbf{64.3 \pm 0.5}$ & $\mathbf{26.8 \pm 3.2}$ & $\mathbf{34.5 \pm 2.9}$ & $\mathbf{28.1 \pm 1.3}$ & $\mathbf{37.4 \pm 1.0}$ \\
  \cmidrule(l){2-11} 
  & \multirow{3}{*}{50\%} & Baseline & $57.1$ & $59.0$ & $33.5$ & $39.8$ & $13.5$ & $21.3$ & $9.7$ & $14.0$ \\ 
  & & KNN & $29.8$ & $38.9$ & $0.0$ & $0.0$ & $0.0$ & $0.0$ & $0.3$ & $0.6$ \\
  \rowcolor{mayablue!30}
  \cellcolor{white}  & \cellcolor{white} & \method{} & $\mathbf{70.1 \pm 0.9}$ & $\mathbf{72.7 \pm 0.7}$ & $\mathbf{49.9 \pm 0.2}$ & $\mathbf{54.2 \pm 0.1}$ & $\mathbf{26.8 \pm 2.3}$ & $\mathbf{36.8 \pm 0.8}$ & $\mathbf{25.9 \pm 0.8}$ & $\mathbf{30.7 \pm 1.6}$ \\ 
  \midrule \midrule  
   & \multirow{3}{*}{0\%} & Baseline & $36.2$ & ${48.9}$ & & & $18.0$ & $29.4$ & & \\
  \multirow{12}{*}{$\mbox{MLP}_v$} & & KNN & $26.2$ & $\mathbf{50.0}$ & & & $7.9 $ & $\mathbf{32.0}$ & & \\
  \rowcolor{mayablue!30}
  \cellcolor{white}  & \cellcolor{white} & \method{} & $\mathbf{38.6 \pm 1.9}$ & $47.6 \pm 3.0$ & & & $\mathbf{21.6 \pm 1.0}$ & $25.0 \pm 1.2$ & & \\
  \cmidrule(l){2-11} 
  & \multirow{3}{*}{30\%} & Baseline & $16.6$ & $19.4$ & $35.9$ & $47.2$ & $17.7$ & $\mathbf{26.1}$ & $19.8$ & $\mathbf{29.0}$                        \\ 
  & & KNN & $\mathbf{21.9}$ & $\mathbf{29.1}$ & $15.3$ & $34.4$ & $7.4$ & $24.4$ & $10.4$ & $31.5$ \\
  \rowcolor{mayablue!30}
  \cellcolor{white}  & \cellcolor{white} & \method{} & $18.1 \pm 0.4$ & $20.9 \pm 0.6$ & $\mathbf{39.3 \pm 1.8}$ & $\mathbf{51.4 \pm 1.1}$ & $\mathbf{21.4 \pm 1.1}$ & $24.6 \pm 1.8$ & $\mathbf{23.5 \pm 1.1}$ & $27.9 \pm 1.3$ \\
  \cmidrule(l){2-11} 
  & \multirow{3}{*}{40\%} & Baseline & $22.6$ & $45.0$ & $19.5$ & $22.8$ & $16.4$ & $\mathbf{27.1}$ & $20.1$ & $\mathbf{26.9}$ \\
  & & KNN & $8.5$ & $42.7$ & $26.8$ & $35.9$ & $10.3$ & $20.3$ & $15.1$ & $28.9$ \\
  \rowcolor{mayablue!30}
  \cellcolor{white}  & \cellcolor{white} & \method{} & $\mathbf{28.5 \pm 3.0}$ & $\mathbf{48.3 \pm 1.7}$ & $\mathbf{21.4 \pm 0.7}$ & $\mathbf{24.5 \pm 0.8}$ & $\mathbf{22.2 \pm 1.1}$ & $26.2 \pm 1.5$ & $\mathbf{22.7 \pm 1.1}$ & $26.8 \pm 1.2$ \\
  \cmidrule(l){2-11} 
  & \multirow{3}{*}{50\%} & Baseline & $17.6$ & $35.2$ & $20.1$ & $29.3$ & $18.6$ & $23.4$ & $17.4$ & $\mathbf{28.6}$ \\ 
  & & KNN & $8.4$ & $33.1$ & $8.5$ & $31.8$ & $8.6$ & $17.3$ & $6.7$ & $18.7$ \\
  \rowcolor{mayablue!30}
  \cellcolor{white}  & \cellcolor{white} & \method{} & $\mathbf{19.6 \pm 0.8}$ & $\mathbf{40.8 \pm 1.5}$ & $\mathbf{23.9 \pm 0.7}$ & $\mathbf{32.8 \pm 1.3}$ & $\mathbf{21.5 \pm 0.9}$ & $\mathbf{24.5 \pm 1.3}$ & $\mathbf{22.0 \pm 0.9}$ & ${25.7 \pm 1.3}$ \\
  \bottomrule 
  \end{tabular} 
  \end{adjustbox}
\end{table*}

\subsection{Extension of NPC}
\label{sec:ext_exp}
Let us analyze the training patterns for the trivial extension to the NPC discussed in \cref{sec:ext} using two random settings, under the unsupverised learning.
\cref{fg:ext_npc} presents the corresponding results.
More experiments were actually carried out, but they showed the same patterns and thus are not reported. 
For both settings involving different datasets and noise settings, the training losses (graphs in the first column) in both cases exhibit a monotonic decreasing tendency.
As shown in the second column, the validation and test losses are randomly fluctuating throughout the epochs in both cases, even though the training loss appears to be properly learned.
Apart from that, the value scope of the training losses is smaller than $1$, while the validation and testing losses are around $50$, illustrating an outstanding discrepancy.
This is a strong indication that the connection between the training and validation losses is lost. 
It further implies that the learning process fails to drive the search for model parameters towards the desired space.
Moreover, we observe from the third column that the micro-F1 for the validation and test sets share the same pattern as their losses. 
All of these aspects suggest that this extension is incapable of appropriately learning the distribution of the true labels.

\subsection{Unsupervised Learning Results}
\label{sec:unsup_results} 
Since the extension of NPC has been shown to be an invalid approach in the previous subsection, we exclude it from the following empirical comparisons.
Following~\cite{Bae2022from}, we also implemented a K-nearest neighbor (KNN) method for comparison. 
Since KNN is a deterministic method, we ran it only once for each configuration.
The hyperparameter K is defaulted to 5.
\cref{tb:unsup} exhibits the comparison results for the unsupervised learning using \vocseven{} and \tomato{}.
Apart from that, setting NR to $0\%$ means that the data attains the level of cleanliness as prepared by the publisher.
The term \enquote{baseline} refers to the base model itself without applying any post-processor, e.g., KNN or \method{}.
 
First, our major finding unveils that \method{} can improve both the macro- and macro-F1 in most cases. 
Even for HLC, which is a method improving upon the noisy label situations, the corresponding \method{} is able to further improve the performance. 
It shows that for the clean data, \method{} performs slightly worse than the baseline model.
This phenomenon is expected as our method concentrates highly on the noisy situations. 
We also observe that the improvement of \method{} is even higher when the noise rate increases. 
It implies that our method can well address the noisy label situations.
The \tomato{} dataset is a tiny dataset which may contain a great amount of variation.
In the paper~\cite{gehlot2023tomato}, which published the data, it is evident that the variation in experimental results was significant across the models, despite only slight changes in parameter size under the same model.
For this complex data, some of \method{} outcomes show a slight decrease of the micro-F1; however, it could achieve a greater increase in the macro-F1, which we also deem an overall improvement.
All the results matching this pattern were found when the base model is LeViT, which is smaller than ResNet-50. 
For \tomato{}, HLC leads to a drop in performance given its base model of ADDGCN, even though HLC is proven working well for many other cases.
However, our approach is still able to boost the performance on top of ADDGCN. 
 
Second, the KNN method could perform the best in rare cases.
It appears to be a random method that cannot be guaranteed to perform well, so that it obtains poor performance in many cases.
Even by heuristics, the KNN is also not a suitable approach given the complexity of the multilabel classification cases. 
We emphasize that the zeros reported in the table are not mistaken. 

 
\begin{figure*}[!t]
  \centering 
  \subfloat{
    \includegraphics[width=0.4\textwidth]{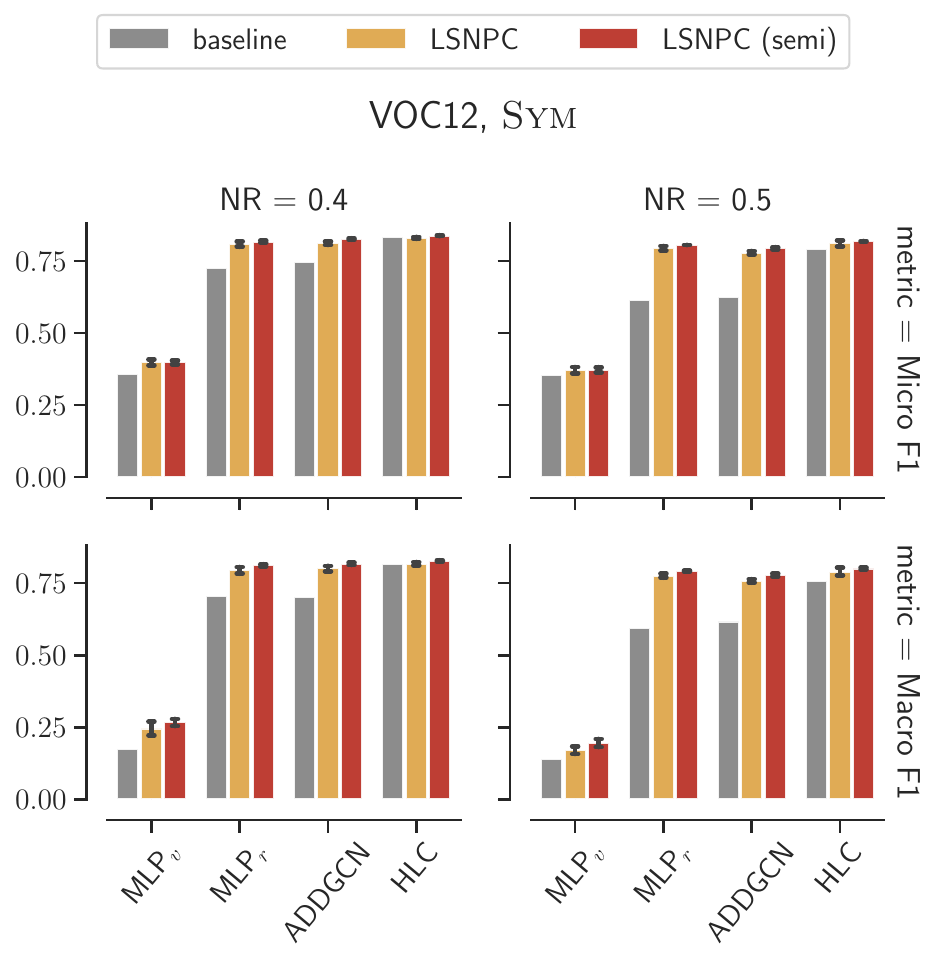}
  }
  \quad \quad
  \subfloat{
    \includegraphics[width=0.4\textwidth]{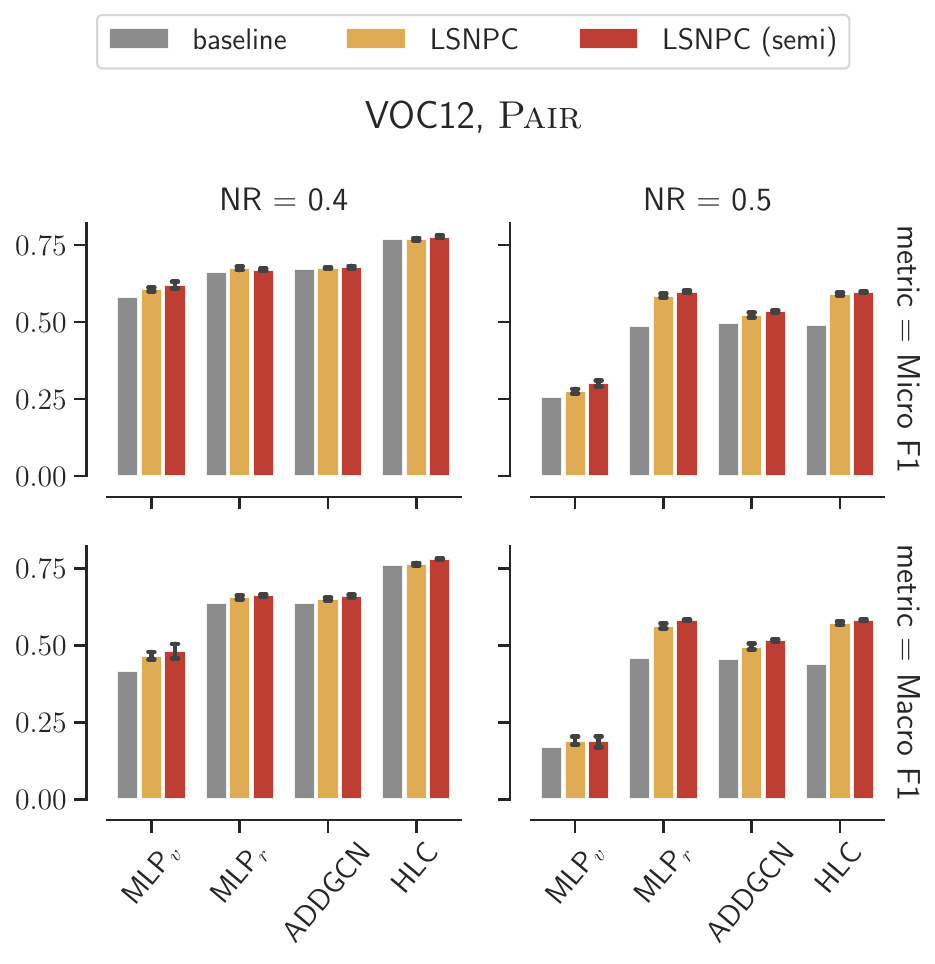}
  }
  \\
  \subfloat{
    \includegraphics[width=0.4\textwidth]{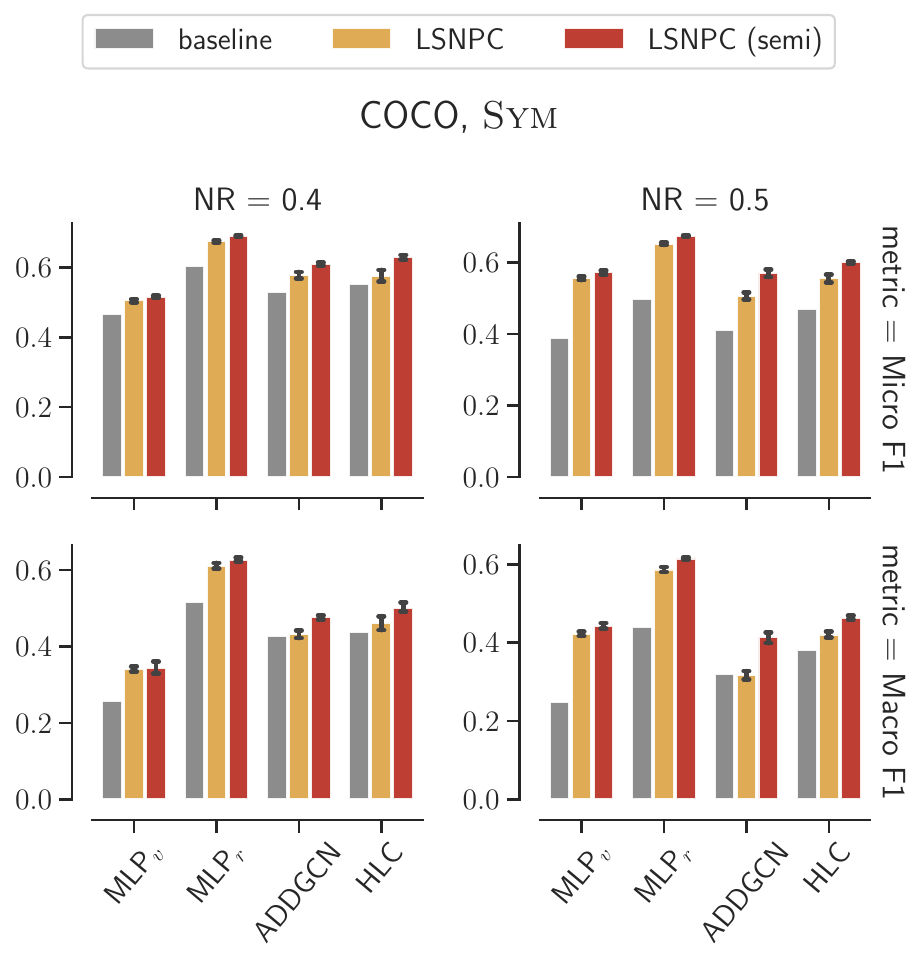}
  }
  \quad \quad
  \subfloat{
    \includegraphics[width=0.4\textwidth]{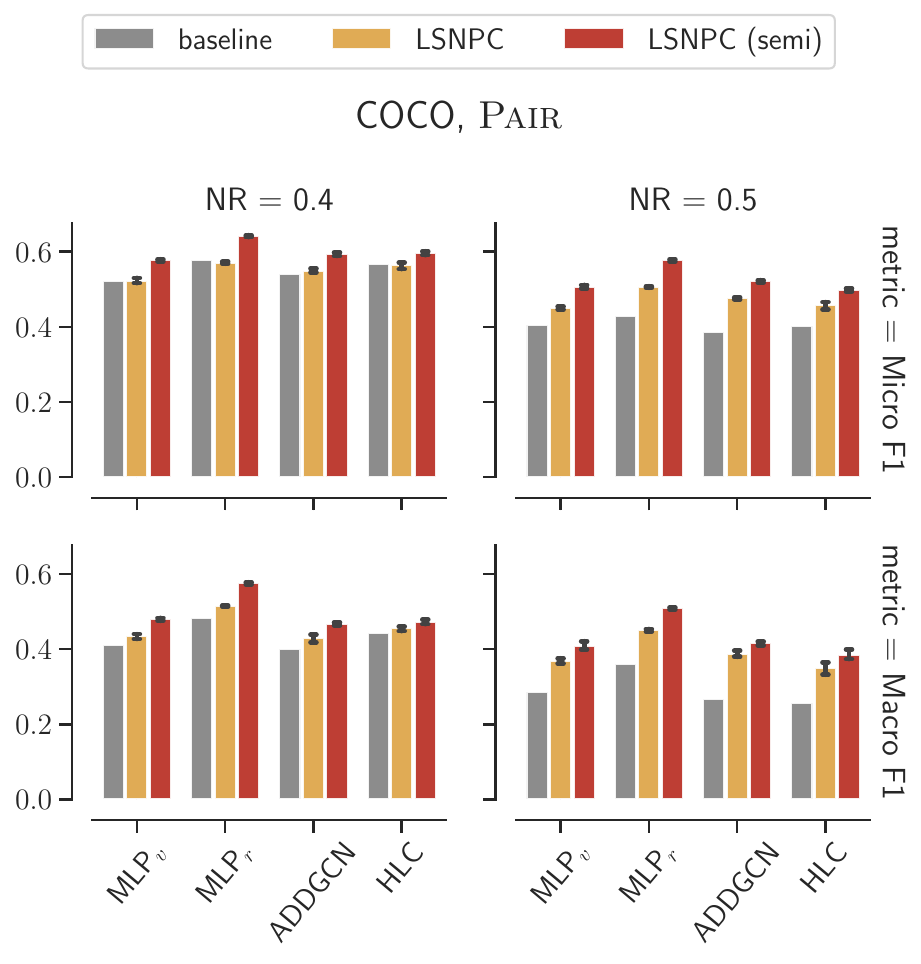}
  }
  \caption{Unsupervised and semi-supervised learning results for \voctwelve{} and \coco. The standard deviations are presented as the error bars.} \label{fg:semi_sup}
\end{figure*}

\subsection{Semi-Supervised Learning Results} 
\label{sec:semisup_results}
As the KNN does not work properly, it is no longer included in the experiments here.
\cref{fg:semi_sup} illustrates the outcomes focusing on the NR within $\{30\%, 40\%\}$ for the two noise settings.
The results for \voctwelve{} and \coco{} are respectively demonstrated on the first and second row. 
For each sub-graph, the micro-F1 and macro-F1 are presented in the first and second row respectively.
All semi-supervised learning algorithms were trained for 10 and 5 epochs respectively for the \voctwelve{} and \coco{} dataset, less than that for the unsupervised learning. 
We omitted certain experiments in which clean data was utilized to fine-tune the baseline models that had been trained using noisy labels. 
The reason was that, all the results turned out to be even worse for every single baseline model.

To summarize, the unsupervised \method{} consistently outperform the single model (which again is the baseline).
This is also a strong indication of the effectiveness of our unsupervised \method{}.
Furthermore, the semi-supervised \method{} can outperform the unsupervised fashion in all cases.
In a few cases, e.g., for \voctwelve{} and \sym{}, HLC performs rather well compared with other baseline models; however, \method{} with either unsupervised and semi-supervised paradigm is still able to improve on top of it.
Interestingly, for the \coco{} dataset with both types of noises, our \method{} boosts the performance of ADDGCN, under NR of $50\%$, to a level that is close to the level of performances of any model, under NR of $40\%$, although the base model itself does not perform very well. 
Matching the patterns found in the previous section analyzing the results, \method{} brings about higher improvements in a more noisy situation.

\begin{table*}[t]
\centering
\caption{Ablation Study on Comparing using Student or Normal Distributions for $\zhat$. The best performances are labeled as bold regarding the mean value.} \label{fg:ablation}
\begin{tabular}{cl cc cc cc cc}
\toprule
 & & \multicolumn{2}{c}{ \vocseven, \pair{}, $\mbox{MLP}_v$}      & \multicolumn{2}{c}{\voctwelve, \sym{}, ADDGCN}      & \multicolumn{2}{c}{\tomato, \pair{}, $\mbox{MLP}_r$}   & \multicolumn{2}{c}{\coco, \sym{}, HLC}    \\   
NR &  & macro-F1       & micro-F1       & macro-F1       & micro-F1       & macro-F1        & micro-F1       & macro-F1        & micro-F1 \\ \midrule
\multirow{3}{*}{30\%} & Baseline   & $65.5$         & $71.3$         & $74.1$         & $74.5$         & $10.8$          & $21.9$   &  $54.3$ & $63.9$     \\
 & GAUSS   & $64.0 \pm 3.1$ & $72.2 \pm 2.0$ & $70.6 \pm 2.1$ & $72.6 \pm 1.0$ & $\mathbf{29.0 \pm 2.2}$  & $\mathbf{38.1 \pm 2.3}$ & ${51.8 \pm 1.3}$ & $\mathbf{60.4 \pm 1.0}$\\
\rowcolor{mayablue!30}
  \cellcolor{white} & \method{} & $\mathbf{66.5 \pm 3.0}$ & $\mathbf{73.5 \pm 1.7}$ & $\mathbf{79.1 \pm 0.3}$ & $\mathbf{80.4 \pm 0.1}$ & $27.9 \pm 2.8 $ & $37.8 \pm 2.4$ & $\mathbf{52 \pm 1.6}$ & ${60.1 \pm 1.3}$ \\ \midrule
 
\multirow{3}{*}{40\%} & Baseline   & $39.7$         & $56.5$         & $68.1$         & $68.6$         & $10.9$          & $18.2$        & $43.7$ 	& $55.4$  \\
 & GAUSS   & $44.1 \pm 3.0$ & $56.9 \pm 1.3$ & $57.0 \pm 0.9$ & $62.4 \pm 1.0$ & $27.0 \pm 2.4$  & $37.4 \pm 1.6$ & ${47.8 \pm 1.2}$  & $\mathbf{60.0 \pm 0.8}$ \\
\rowcolor{mayablue!30}
  \cellcolor{white} & \method{} & $\mathbf{44.9 \pm 1.5}$ & $\mathbf{59.5 \pm 0.9}$ & $\mathbf{77.2 \pm 0.4}$ & $7\mathbf{8.3 \pm 0.6}$ & ${28.1 \pm 1.3}$  & ${37.4 \pm 1.0}$ & $\mathbf{48.9 \pm 2.5}$  & ${59.8 \pm 1.8}$ \\ \midrule

\multirow{3}{*}{50\%} & Baseline   & $16.4$         & $27.4$         & $50.7$         & $59.0$         & $9.7$           & $14$      & $38.1$ 	& $47.1$      \\
 & GAUSS   & $\mathbf{16.4 \pm 1.9}$ & $30.6 \pm 2.1$ & $49.4 \pm 0.4$ & $53.0 \pm 0.5$ & $\mathbf{27.5 \pm 1.8}$  & $\mathbf{32.2 \pm 2.1}$ & $42.6 \pm 0.6$  & $56.4 \pm 0.5$ \\
  \rowcolor{mayablue!30}
  \cellcolor{white} & \method{} & $16.1 \pm 2.0$ & $\mathbf{33.0 \pm 1.9}$ & $\mathbf{70.7 \pm 0.6}$ & $\mathbf{76.8 \pm 0.3}$ & $25.9 \pm 0.8$  & $30.7 \pm 1.6$ & $\mathbf{43.0 \pm 1.7}$  & $\mathbf{56.4 \pm 1.7}$ \\ \midrule
\end{tabular} 
\end{table*}

\subsection{Sensitivity Analysis}
\label{sub:sens}

\begin{figure}[!t]
  \centering
  \subfloat[]{
  \includegraphics[width=0.45\textwidth]{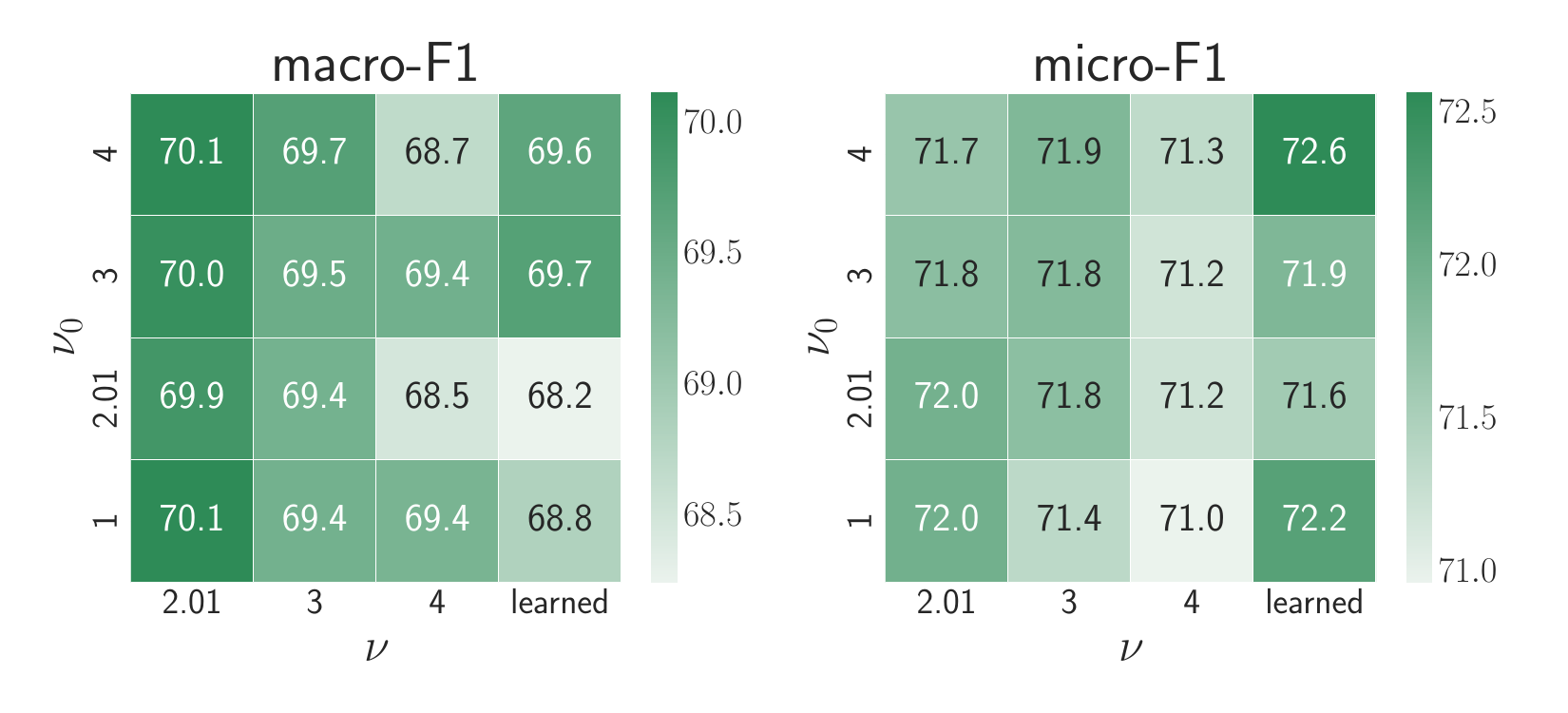}
  } \\
  \subfloat[]{
  \includegraphics[width=0.45\textwidth]{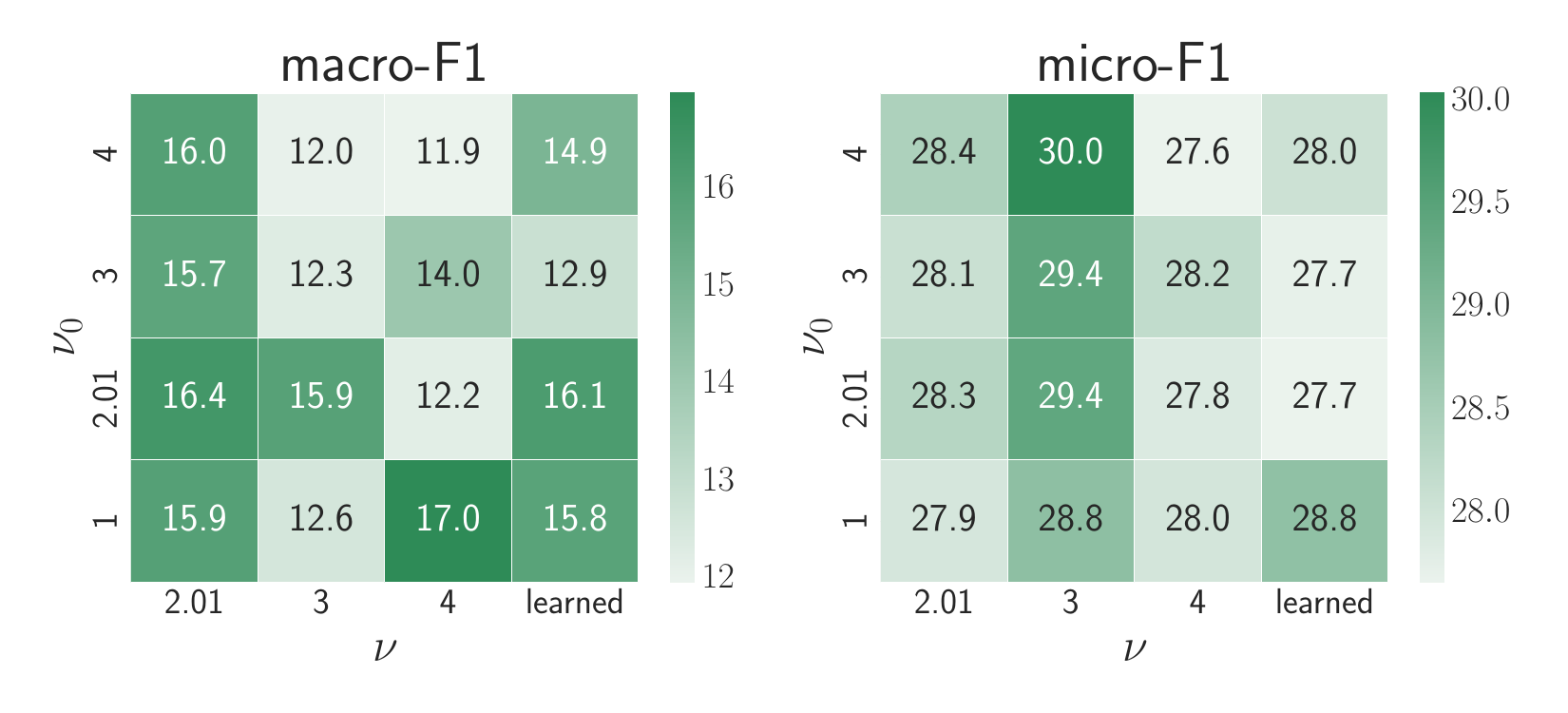}
  }
  \caption{
  	Empirical results of the sensitivity analysis.
  	(a) Sensitivity analysis of choices over $\nu_0$ and $\nu$ on \vocseven. The base model is chosen to be HLC. The noise setting is \pair{} with a rate of 0.4.
  	(b) Sensitivity analysis of choices ove{}r $\nu_0$ and $\nu$ on Tomato. The base model is chosen to be $\mbox{MLP}_v$. The noise setting is \sym{} with a rate of 0.3.
  } \label{fg:sens} 
\end{figure}

\cref{fg:sens} presents the sensitivity analysis results, in which each number is the averaged outcome of five repetitions. 
The value of $2.01$ for $\nu$ aligns to the requirements in our theoretical analysis (\cref{thm:z_bound}) that $\nu$ and $\nu_0$ must be greater than 2.
Hence, we arbitrarily set $\nu$ and $\nu_0$ to $2.01$ when considering a value close to 2.
The experiments for the field of \enquote{learned} were appended after other experiments were finished. 
Under the setting of \enquote{learned}, we therefore define a function $\nu_{\theta}(\cdot)$ which reforms~\cref{eq:sup_zhat} to 
\begin{align*}
q(\zhat \mid \x, \yhat)
&= \mathrm{Student}(\zhat; \mu_{\theta}(\x, \yhat), \mathrm{diag}(\sigma^2_{\theta}(\x, \yhat)), \nu_{\theta}(\x, \yhat)) \,.
\end{align*}
Given a specified network architecture $\mathrm{Net}(\x, \yhat)$, we define $\nu_{\theta}(\x, \yhat)$ as
\begin{align}
\nu_{\theta}(\x, \yhat) = \mathrm{ReLU}(\mathrm{Net}(\x, \yhat)) + 1  \, 
\end{align}
where $\mathrm{ReLU}(\cdot)$ guarantees the non-negativity of $\nu$ and $+1$ is the minimal value of the degree of freedom in the Student distribution. 
 
With regard to the \voctwelve{} dataset, one may observe that the hyperparameters $\nu_0$ and $\nu$ are generally robust with regard to the three performance. 
The under performing settings are merely slightly worse. 
We observe that the performances were decreased for the most cases of $\nu > \nu_0$. 
Although learnable $\nu$ can be the best performer with regard to micro-F1, the fixed hyperparameters with $\nu = 2.01$ locks the best macro-F1 performance provided that its micro-F1 are also sufficiently good.
This might imply that the learnable $\nu$ tend to push the model to focus more on the major labels while a fixed hyperparameter better regularizes \method{} to also attend to the minor labels. 
Apart from that, it shows that for a learnable setting of $\nu$, the choice of $\nu_0$ becomes more sensitive. 

While for \tomato{}, the learnable setting of $\nu$ seems to have a sensitive performance in terms of the macro-F1. 
Also, despite of the best performance of $\nu = 3$ in micro-F1, the macro-F1 heavily underperformed for most settings.
For practitioners who own the capacity of running hyperparameter optimization, one may be able to search for a nice pair of setting that could gain further improvements. 
Otherwise, $(\nu_0, \nu) = (2.01, 2.01)$ is a pair of consistent performing hyperparameters that we could always start our experiments with.

\subsection{Ablation Study}
\label{sub:as}
In the ablation study, we compared the usage of Student and Normal distributions for the latent variable $\zhat$.  
For the Normal distribution setting, we accordingly modify the relevant proposal distributions such that 
\begin{align}
q(\zhat \mid \x, \yhat) = \mathrm{Normal}(\zhat; \mu_{\theta}(\x, \yhat), \mathrm{diag}(\sigma_{\theta}^2(\x, \yhat))) \,.
\end{align} 
This approach is denoted by GAUSS. 
We randomly selected a few configurations for the ablation studies and present the results in~\cref{fg:ablation}. 	
For certain cases, the Student distribution excels the Normal distribution significantly.

Meanwhile, for the cases that Normal performs better, the differences are considerably close, in particular considering that the outcomes have been multiplied by 100. 
It assures that applying the Student distribution is at least as good as the Normal distribution. 
However, the findings suggest that the Normal distribution may outperform in certain scenarios, and practitioners might consider conducting a preliminary experiment to determine the most suited distribution if optimal outcomes are desired.

\subsection{GradCAM Analysis}
To understand what information \method{} focuses on while classifying multilabel images, we use GradCAM method~\cite{selvaraju2017grad}. The first three columns of
\cref{fg:gradcam} shows three GradCAM examples on the \vocseven{} dataset with $\mbox{MLP}_{r}$ (first row) and \method{} (second row), where $\mbox{MLP}_{r}$ is used as the pre-trained classifier. 
The labels at the top of each column indicate ground truth labels for each image, with labels highlighted in blue representing those that were initially missed by the pre-trained classifier but later corrected by \method{}. 
As can be observed from the first column, \method{} concentrates on features relevant to the label \enquote{pottledplant} when this label has not been assigned by the pre-trained classifier.
For the second and third images, \method{} focuses on regions relevant to labels missing from the predictions of $\mbox{MLP}_{r}$, such as the \enquote{dining table} in the third image. 
This behavior can be attributed to the fact that, as \method{} takes both the image features and the predictions of $\mbox{MLP}_{r}$ as inputs, it does not necessarily focus on regions associated with correctly assigned labels but instead directs focus to other regions that might be relevant to missing labels. 
The last two columns show certain misclassified labels using $\mbox{MLP}_{r}$, highlighted in red. 
\method{} is capable of effectively constricting the regions to focus on the ground truth labels, which facilitates the removal of the mis-classified ones. 
\begin{figure}[!t]
\centering
\includegraphics[width=0.48\textwidth]{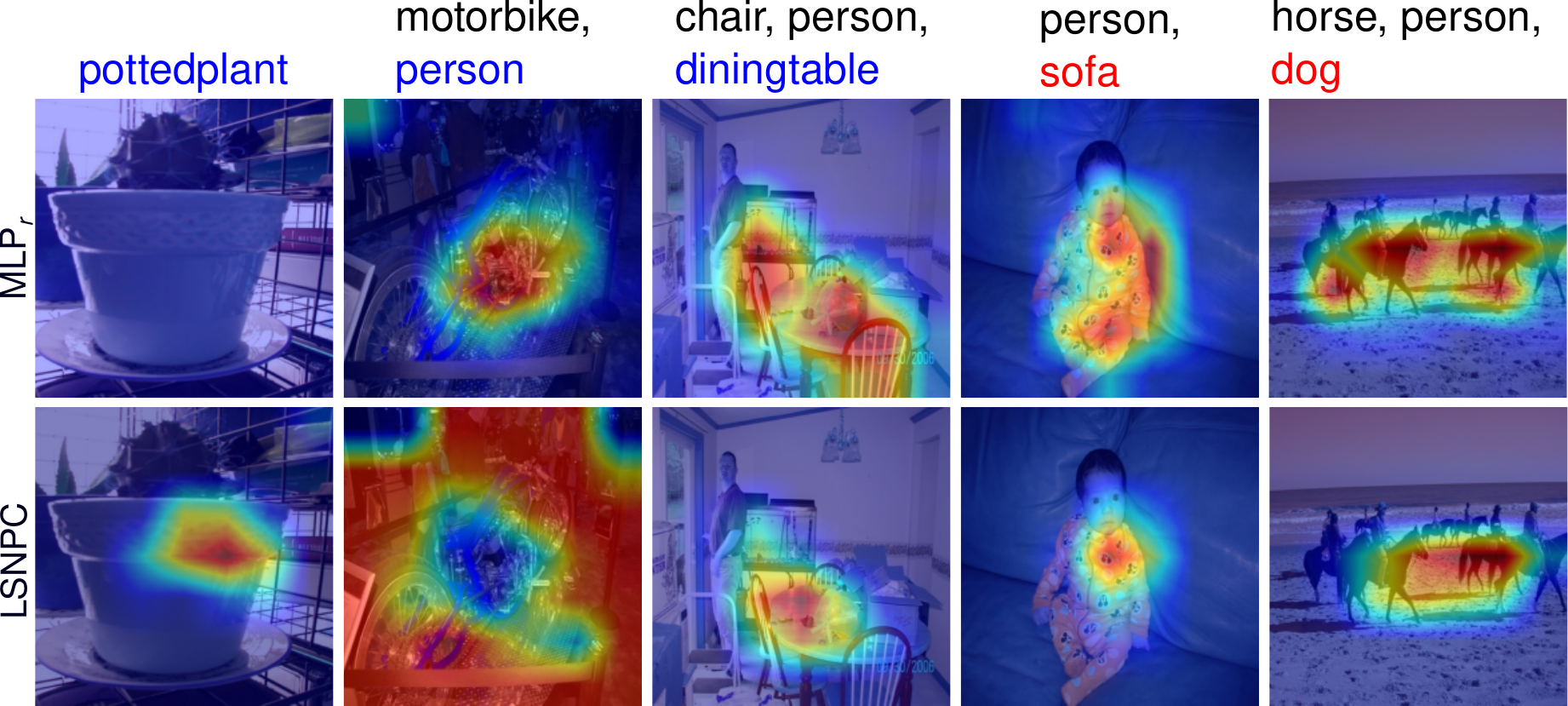}	
\caption{
	GradCAM examples on \vocseven{} with $\mbox{MLP}_{r}$ and \method{} using $\mbox{MLP}_{r}$ as the pre-trained classifier. Ground truth labels are on top of each column, and the labels highlighted in {\color{blue} blue} and {\color{red} red} indicate those missed and mis-classified from the classifications of $\mbox{MLP}_{r}$ respectively, but then corrected by \method{}.
} 
\label{fg:gradcam} 
\end{figure}

\section{Conclusion}
\label{sec:conclusion}
In this paper, we present \method{}, a modeling approach that applies Bayesian deep learning to model the noisy label generation.
Unlike all the exiting approaches, we posit that the label noise is actually generated through a shift of the latent variable of the labels in the space.
It inherently addresses the problem of the label sparsity and the label correlations in the multilabel setting since it utilizes the latent variable, which is the clustered feature, for reconstructing the true labels.
The framework supports both unsupervised and semi-supervised learning paradigm. 
We theoretically analyze that the discrepancy between the true labels and estimated labels are upper bounded.
Empirically, we demonstrate that \method{} is capable of improving the base models under a variety of settings.
Apart from that, we carry out qualitative analysis on certain cases to illustrate how \method{} corrects the predictions.
\method{} is shown to be more beneficial when the label noise is greater, which is a more common scenario for small businesses handling large datasets.
We envision this as a step towards the ultimate goal to rectify the noisy labels in data rather than model predictions, in a robust manner. 
For future work, one may explore the methodology of determining the neural network architectures that better realize the potential of this approach. 
Furthermore, it is also crucial to improve the ability to deal with low-level noise cases.

\section*{Acknowledgments}
W. Huang has been supported by the Doctoral Research Initiation Program of Shenzhen Institute of Information Technology (Grant SZIIT2024KJ001) and Guangdong Research Center for Intelligent Computing and Systems (Grant PT2024C001). 
Q. Li has been partially supported by Natural Science Foundation of Guangdong Province under Grant 2023A1515011845.
J. Liang has been supported by the Shenzhen Science and Technology Program (Grant RCBS20221008093252092), and Guangdong Basic and Applied Basic Research Foundation (Grant 2023A1515110070).
N. J. Hurley has been financially supported by the Science Foundation Ireland (Grant 12/RC/2289\_P2).

{\appendices
\section{Proof of Theorem 1}
The theorem establishes the connection of our objective to the heuristics that we are optimizing the upper bound of the expected KL-divergence of the proposed distribution of true label $\z$ and its corresponding true distribution.

\begin{proof}
Let $\mathcal{H}(\cdot)$ denote the entropy function.
For the unsupervised learning, we obtain
\begin{align*}
&\kl[ q(\z, \zhat \mid \x, \yhat) \mid\mid p(\z, \zhat \mid \x, \yhat)] \\
&= \int \int q(\z, \zhat \mid \x, \yhat) \log \frac{q(\z, \zhat \mid \x, \yhat)}{p(\z, \zhat \mid \x, \yhat)} d\z d\zhat \\
&\ge \int \int q(\z \mid \zhat) q(\zhat \mid \x, \yhat) \log \frac{q(\z \mid \zhat) q(\zhat \mid \x, \yhat)}{p(\z \mid \x, \yhat)} d\z d\zhat \\
&= \int q(\zhat \mid \x, \yhat) \left( \int q(\z \mid \zhat) \log \frac{q(\z \mid \zhat)}{p(\z \mid \x, \yhat)} d\z \right) d\zhat \\
&\quad + \int q(\z \mid \zhat) d\z  \int q(\zhat \mid \x, \yhat)\log q(\zhat \mid \x, \yhat) d\zhat \\
&= \E_{\zhat \sim q(\zhat \mid \x, \yhat)}[\kl[q(\z \mid \zhat) \mid\mid p(\z \mid \x, \yhat) ]] + \underbrace{\mathcal{H}(q(\zhat \mid \x, \yhat))}_{\ge 0} \\
&\ge \E_{\zhat \sim q(\zhat \mid \x, \yhat)}[\kl[q(\z \mid \zhat) \mid\mid p(\z \mid \x, \yhat) ]] \,.
\end{align*}
 
Applying a similar technique, we achieve the results for the supervised version as follows:
\begin{align*}
&\kl[ q(\z, \zhat \mid \x, \y, \yhat) \mid\mid p(\z, \zhat \mid \x, \y, \yhat)] \\
&= \int \int q(\z, \zhat \mid \x, \yhat) \log \frac{q(\z, \zhat \mid \x, \yhat)}{p(\z, \zhat \mid \x, \yhat)} d\z d\zhat \\
&\ge \E_{\zhat \sim q(\zhat \mid \x, \yhat)}[\kl[q(\z \mid \zhat, \x, \yhat) \mid\mid p(\z \mid \x, \yhat) ]] \,.
\end{align*}
The proof completes here.
\end{proof}

\section{Proof of Theorem 2}
\begin{proof} 
Our first step is to show 
\begin{align}
&\kl[ q(\z \mid \x, \yhat = \y_1) \mid\mid q(\z \mid \x, \yhat = \y_0)] \notag \\
&\le \kl[ q(\zhat \mid \x, \yhat = \y_1) \mid\mid q(\zhat \mid \x, \yhat = \y_0)] \,
\end{align}
For this inequality, we resort to the data processing inequality (DPI)~\cite{thomas2006elements}.
One form of DPI is as follows: given $P(X)$ and $Q(X)$ which share the same transition function mapping $X \to Y$, the data processing inequality with regard to KL-divergence is 
\begin{align}
  \kl[P(X) \mid \mid Q(X)] \geq \kl[P(Y) \mid \mid Q(Y)] \,. 
\end{align}  
We first notice that the distribution $q(\z \mid \x,  \yhat=\y_0)$ and $ q(\z \mid \x, \yhat=\y_1)$ might differ in the VAEs, since their parameter values could diverge given different inputs of $\yhat$ to the decoder function.
By replacing $X$ by $\zhat \mid \x, \yhat$ and $Y$ by $\z \mid \x, \yhat$, we obtain 
\begin{multline}
\kl[ q(\z \mid \x, \yhat=\y_1) \mid\mid q(\z \mid \x, \yhat=\y_0)] \\
\le \kl[ q(\zhat \mid \x, \yhat=\y_1) \mid\mid q(\zhat \mid \x, \yhat=\y_0)] \,,
\end{multline}  
regardless of the transition function being either $q(\z \mid \zhat)$ or $q(\z \mid \x, \y, \zhat)$ as long as they are applied consistently.
   
With regard to the second inequality in~\cref{eq:gap_ineq}, we apply the result in~\cref{lem:student_bound} and are able to complete the proof.
\end{proof}

Before presenting~\cref{lem:student_bound}, we first summarize the theoretical results of KL-divergence between two multivariate Student distributions from a recent work~\cite{huang2019novel}, shown in~\cref{thm:student_kl}.
\begin{theorem}
\label{thm:student_kl}
Given two multivariate Student distributions, respectively, defined as 
\begin{align*}
p_1(\x) &= \mathrm{Student}(\mu_1, \Sigma_1, \nu_1)  \\
p_2(\x) &= \mathrm{Student}(\mu_2, \Sigma_2, \nu_2) \,.
\end{align*}
Let us fix $\nu_1 > 2$ and denote $\tilde{\Sigma}_1 = \frac {\nu_1}{\nu_1 - 2}\Sigma_1$.
\begin{align}
&\kl [ p_1(\x) \mid \mid p_2(\x) ] \notag \\
&\le \frac 1 2 \log \frac{\mathrm{det}(\Sigma_2)}{\mathrm{det}(\Sigma_1)} + \frac 1 2 \log \frac{\nu_2}{\nu_1} +  \frac 1 2 \Gamma(\nu_2/2) - \frac 1 2 \Gamma(\nu_1/2) \notag \\
&\quad + \frac 1 2 \Gamma((\nu_2+m)/2) - \frac 1 2 \Gamma((\nu_1+m)/2) \notag \\
&\quad - \frac{\nu_1 + m}{2}\{\Psi({(\nu_1 + m)}/{2}) - \Psi(\nu_1/ 2)\} \notag \\
&\quad + \frac {\nu_2 + m} 2 \log \left\{ 1 + \frac{1}{\nu_2} \mathrm{tr}(\Sigma_2^{-1} \tilde{\Sigma}_1)  \right. \notag \\
&\qquad \left. + \frac{1}{\nu_2} \mathrm{tr}(\Sigma_2^{-1} (\mu_1 - \mu_2)(\mu_1 - \mu_2)^T) \right\}
\end{align}
where $\Gamma(\cdot), \Psi(\cdot)$, and $\mathrm{tr}(\cdot)$ are respectively the gamma function, digamma function, and trace of the matrix.  
\end{theorem}

\begin{corollary}
\label{corol:student_kl}
For our case, the two distributions share the same value of $\nu$, i.e., $\nu=\nu_1=\nu_2$.
\cref{thm:student_kl} can be simplified to
\begin{align}
&\kl [ p_1(\x) \mid \mid p_2(\x) ] \notag \\
&\le \frac 1 2 \log \frac{\mathrm{det}(\Sigma_2)}{\mathrm{det}(\Sigma_1)} - \frac{\nu + m}{2}\left\{\Psi\left( \frac{\nu + m}{2} \right) - \Psi\left(\frac{\nu}{2} \right ) \right\} \notag \\
&\quad + \frac {\nu + m} 2 \log \left\{ 1 + \frac{1}{\nu} \mathrm{tr}(\Sigma_2^{-1} \tilde{\Sigma}_1)  \right. \notag \\
&\qquad \left. + \frac{1}{\nu} \mathrm{tr}(\Sigma_2^{-1} (\mu_1 - \mu_2)(\mu_1 - \mu_2)^T) \right\} \,.
\end{align}
\end{corollary}
\begin{proof}
This is a straightforward extension of~\cref{thm:student_kl} by knowing $\nu_1 = \nu_2 = \nu$.
\end{proof}

\begin{lemma}
\label{lem:student_bound}
Consider Assumptions 1 to 3 satisfy. 
Fixing $\alpha = \frac{\sqrt{\nu \lambda}}{L}$, we have 
\begin{align}
\kl[ q(\zhat \mid \x, \y_1) \mid\mid q(\zhat \mid \x, \y_0)] 
\le C_1 + C_2 \Delta(\y_0, \y_1) \label{eq:gap_ineq} \,, 
\end{align}
where 
\begin{align*}
  C_1 &= \frac{\nu+m}{2} \left\{ \frac{M \sqrt{m} \alpha}{2(\nu-2) }  - \Psi\left( \frac{\nu + m}{2} \right) + \Psi\left(\frac{\nu}{2} \right) \right\} \\
  C_2 &= \frac{m M}{2e} + \frac{(\nu+m) \sqrt{m}}{2\alpha}  \,.
\end{align*}
\end{lemma}

\begin{proof}
For convenience, given an input $\x$ and its labels $\y$, we denote 
\begin{align*}
\mu(\y) &\coloneqq \mu_{\theta}(\x, \y) \\
\Sigma(\y) &\coloneqq \mathrm{diag}\{ \sigma^2_{\theta}(\x, \y) \} \, 
\end{align*}
which omits $\x$ as $\x$ is fixed in our studied case. 
It follows that, taking into account the diagonal property, the determinant is simplified to
\begin{align}
  \mathrm{det}(\Sigma(\y)) &= \prod_{j=1}^{m} \Sigma(\y)_{jj} = \prod_{j=1}^{m} \sigma^2_{\theta}(\x, \y)_{j} \label{eq:det} \,.
\end{align}
In addition, the trace of a matrix is the sum of its diagonal elements, such that 
\begin{align}
  \mathrm{tr}(\Sigma(\y)) &= \sum_{j=1}^{m} \Sigma(\y)_{jj} = \sum_{j=1}^{m} \sigma^2_{\theta}(\x, \y)_{j} \label{eq:trace} \,.
\end{align}

Employing the results in~\cref{corol:student_kl}, we obtain 
\begin{align}
&\kl[ q(\zhat \mid \x, \y_1) \mid\mid q(\zhat \mid \x, \y_0)] \notag \\
&\le \frac 1 2 \log \frac{\mathrm{det}(\Sigma(\y_0))}{\mathrm{det}(\Sigma(\y_1))} - \frac{\nu + m}{2}\left\{\Psi\left( \frac{\nu + m}{2} \right) - \Psi\left(\frac{\nu}{2} \right ) \right\} \notag \\
&\quad + \frac {\nu + m} 2 \log \left\{ 1 + \frac{1}{\nu} \mathrm{tr}(\Sigma(\y_0)^{-1} \tilde{\Sigma}(\y_1))  \right. \notag \\
&\qquad \left. + \frac{1}{\nu} \mathrm{tr}\left\{\Sigma(\y_0)^{-1} (\mu(\y_1) - \mu(\y_0))(\mu(\y_1) - \mu(\y_0))^T \right\} \right\} \label{eq:our_kl} \,.
\end{align}
Hence, combining \cref{prop:det,prop:trace},~\cref{eq:our_kl} becomes
\begin{align}
&\kl[ q(\zhat \mid \x, \y_1) \mid\mid q(\zhat \mid \x, \y_0)] \notag \\
&\le \frac{\nu+m}{2} \left\{ \frac{M \alpha}{2(\nu-2)}  - \Psi\left( \frac{\nu + m}{2} \right) + \Psi\left(\frac{\nu}{2} \right) \right\} \notag \\
&\quad + \left( \frac{m M}{2e} + \frac{(\nu+m) \sqrt{m}}{2\alpha}  \right) \Delta(\y_0, \y_1) \,.
\end{align}
It suffices to complete the proof with 
\begin{align*}
C_1 &= \frac{\nu+m}{2} \left\{ \frac{M \sqrt{m} \alpha}{2(\nu-2)}  - \Psi\left( \frac{\nu + m}{2} \right) + \Psi\left(\frac{\nu}{2} \right) \right\} \\
C_2 &= \frac{m M}{2e} + \frac{(\nu+m) \sqrt{m}}{2\alpha} \,.
\end{align*} 

\end{proof}

\begin{proposition}
\label{prop:det}
Given Assumptions 1 to 3 hold true, the term $\frac 1 2 \log\frac{\mathrm{det}(\Sigma(\y_0))}{\mathrm{det}(\Sigma(\y_1))}$ can be upper bounded as follows.
\begin{align}
\frac 1 2 \log\frac{\mathrm{det}(\Sigma(\y_0))}{\mathrm{det}(\Sigma(\y_1))} \le \frac{m M }{2e} \Delta(\y_0, \y_1) \,.
\end{align}
\end{proposition}

\begin{proof} 
Based on~\cref{eq:det}, one may be able to show 
\begin{align}
\frac 1 2 \log\frac{\mathrm{det}(\Sigma(\y_0))}{\mathrm{det}(\Sigma(\y_1))}
&= \frac 1 2 \log \prod_j \frac{\Sigma(\y_0)_{jj}}{\Sigma(\y_1)_{jj}} \notag \\
&\le \frac 1 2 \log\{ M \Delta(\y_0, \y_1) \}^{m} \notag \\  
&\le \frac{m M }{2e} \Delta(\y_0, \y_1) \,, 
\end{align} 
given $\log(x) \le e^{-1} x$ for any $x>0$. 
\end{proof}

\begin{proposition}
\label{prop:trace_terms}
Suppose Assumptions 1 to 3 satisfy.
The inequalities
\begin{align}
\frac{1}{\nu} \mathrm{tr}(\Sigma(\y_0)^{-1} \tilde{\Sigma}(\y_1)) \le \frac{M m}{\nu - 2} \Delta (\y_0, \y_1) 
\end{align}
and 
\begin{multline}
\mathrm{tr}(\Sigma(\y_0)^{-1} (\mu(\y_0) - \mu(\y_1)) (\mu(\y_0) - \mu(\y_1))^T ) \\
\le \frac{m L^2}{\lambda} \Delta(\y_0, \y_1)^2  
\end{multline}
hold.
\end{proposition}

\begin{proof}  
We then analyze the first term and derive
\begin{align}
\frac{1}{\nu} \mathrm{tr}(\Sigma(\y_0)^{-1} \tilde{\Sigma}(\y_1))  
&=\frac 1 \nu \mathrm{tr} \left(\Sigma({\y_0})^{-1} \frac{\nu}{\nu-2} \Sigma(\y_1) \right) \notag \\
&=\frac 1 \nu \frac{\nu}{\nu-2} \mathrm{tr} \left(\Sigma({\y_0})^{-1} \Sigma(\y_1) \right) \notag \\
&=\frac{1}{\nu - 2} \mathrm{tr} \left\{ [\sigma^2(\x, \y_0)^{-1} \m{I}] [\sigma^2(\x, \y_1) \m{I}] \right\} \notag \\ 
&= \frac{1}{\nu - 2} \max_{j=1, \dots, m} \sum_{j=1}^m \frac{\sigma^2(\x, \y_0)_{j}}{\sigma^2(\x, \y_1)_{j}} \notag \\ 
&\le \frac{m}{\nu - 2} \max_{j=1, \dots, m} \frac{\sigma^2(\x, \y_0)_{j}}{\sigma^2(\x, \y_1)_{j}} \notag \\
&\le \frac{Mm}{\nu - 2} \Delta (\y_0, \y_1) \label{eq:first_term} \,. 
\end{align} 
Next, we bound $\mathrm{tr}\left(\Sigma(\y_0)^{-1} (\mu(\y_0) - \mu(\y_1)) (\mu(\y_0) - \mu(\y_1))^T \right)$ by 
\begin{multline}
\mathrm{tr}\left \{ \Sigma(\y_0)^{-1} (\mu(\y_0) - \mu(\y_1)) (\mu(\y) - \mu(\y_1))^T \right \} \\
\le \mathrm{tr}(\Sigma(\y_0)^{-1}) \mathrm{tr}\left \{ (\mu(\y_0) - \mu(\y_1)) (\mu(\y_0) - \mu(\y_1))^T \right \} \,.
\end{multline}
since a positive diagonal matrix and $\m{w} \m{w}^T$ for any $\m{w} \in \mathbb{R}^{a \times b}$ are both positive semi-definite; thus, the trace can be decomposed.
We hence derive
\begin{align*}
&\mathrm{tr} \left\{ (\mu(\y_0) - \mu(\y_1)) (\mu(\y_0) - \mu(\y_1))^T \right\} \\
&= \sum_{j=1}^{m} (\mu(\y_0) - \mu(\y_1))^2_{jj} \\
&= \left \| \mu(\y_0) - \mu(\y_1) \right \|^2_2 \\  
&\le L^2 \Delta(\y_0, \y_1)^2 \,.
\end{align*} 
On the other hand, for any given $\y$, 
\begin{align}
\mathrm{tr}(\Sigma(\y_0)^{-1}) 
&= \sum_{j=1}^{m} \frac{1}{\Sigma(\y_0)_{jj}} \le \frac{m}{\lambda} \,.
\end{align} 
Combining the above two inequalities, we get 
\begin{multline}
\mathrm{tr} \left\{ \Sigma(\y_0)^{-1} (\mu(\y_0) - \mu(\y_1)) (\mu(\y_0) - \mu(\y_1))^T  \right\} \\
\le \frac{m L^2}{\lambda} \Delta(\y_0, \y_1)^2 \label{eq:second_term}\,.
\end{multline}
\cref{eq:first_term,eq:second_term} conclude the proof.
\end{proof}

\begin{corollary}
\label{prop:trace}
Suppose Assumptions 1 to 3  satisfy. 
The following inequality holds.
\begin{align}
&\frac {\nu + m} 2 \log \left\{ 1 + \frac{1}{\nu} \mathrm{tr}(\Sigma(\y_0)^{-1} \tilde{\Sigma}(\y_1))  \right. \notag \\
&\qquad \left. + \frac{1}{\nu} \mathrm{tr}(\Sigma(\y_0)^{-1} (\mu(\y_0) - \mu(\y_1))(\mu(\y_0) - \mu(\y_1))^T \right\} \notag \\
&\le \frac{(\nu + m) \sqrt{m}}{2} \left( \frac{M \alpha}{2(\nu-2)} + \frac{1}{\alpha} \Delta(\y_0, \y_1) \right)\,,
\end{align}
where $\alpha = \frac{\sqrt{\nu \lambda}}{L}$.
\end{corollary}
\begin{proof}
First, we derive the inequality of $\log$ that we would apply throughout our proof. 
With $x > -1$, 
\begin{multline}
\log (1+x) 
\le \frac{x}{\sqrt{x+1}}  
= \frac{\sqrt{x} \sqrt{x}}{\sqrt{x+1}} \\
\le \frac{\sqrt{x} \sqrt{x+1}}{\sqrt{x+1}} 
= \sqrt{x} \,.
\end{multline}

Extending from~\cref{prop:trace_terms} with the above inequality, we therefore achieve  
\begin{align*}
&\frac{\nu+m}{2} \log \left\{ 1 + \frac{M m \Delta(\y_0, \y_1)}{\nu-2} + \frac{m L^2 \Delta(\y_0, \y_1)^2}{\nu \lambda} \right\} \\
&\le \frac{\nu+m}{2} \sqrt{ \frac{M m \Delta(\y_0, \y_1)}{\nu-2} + \frac{m L^2 \Delta(\y_0, \y_1)^2}{\nu \lambda} } \\
&= \frac{(\nu+m) \sqrt{m}}{2} \\
&\quad \times \sqrt{ \left(\frac{M \sqrt{\nu \lambda}}{ 2 (\nu - 2) L} + \frac{L}{\sqrt{\nu \lambda}} \Delta(\y_0, \y_1) \right)^2 - \mathrm{CONST.} } \\
&\le \frac{(\nu+m) \sqrt{m}}{2} \left( \frac{M \sqrt{\nu \lambda} }{2(\nu-2) L} + \frac{L}{\sqrt{\nu \lambda}}\Delta(\y_0, \y_1) \right) \\&= \frac{(\nu+m) \sqrt{m}}{2} \left( \frac{M \alpha}{2(\nu-2)} + \frac{1}{\alpha} \Delta(\y_0, \y_1) \right) \,, 
\end{align*} 
concerning $\alpha = \frac{\sqrt{\nu \lambda}}{L}$.
Moreover, any number is smaller than or equal to the result when it is decreased by a non-negative constant (i.e., CONST.).
The proof completes here. 
\end{proof}

\section{Multivariate Normal Replacing Multivariate Student as the Proposal Distribution}
\label{ap:normal_proposal}
Let us consider the case that the proposal distributions $q(\zhat \mid \x, \yhat=\y_0)$ and $q(\zhat \mid \x, \yhat=\y_1)$ are assumed to be multivariate Normal distributions, such that 
\begin{align}
q(\zhat \mid \x, \yhat=\y_0) &\coloneqq N(\mu(\y_0), \Sigma(\y_0)) \label{eq:normal_proposal0} \\ 
q(\zhat \mid \x, \yhat=\y_1) &\coloneqq N(\mu(\y_1), \Sigma(\y_1)) \label{eq:normal_proposal1}\,.
\end{align}
With this setting replacing the proposal being Student distributions, we have the following corollary.
\begin{corollary}
Assume the proposal distributions follow the multivariate Normal, such that~\cref{eq:normal_proposal0,eq:normal_proposal1} hold.
We get  
\begin{multline}
  D_{KL} [ q(\zhat \mid \x, \yhat=\y_1) \mid \mid q(\zhat \mid \x, \yhat=\y_0)] \\ 
  = O \left( m \Delta(\y_0, \y_1)^2 \right) \,.
\end{multline}
\end{corollary}

\begin{proof}
As known, the KL-divergence for two multivariate Normal distributions has a closed form.
That is, the KL-divergence is 
\begin{align}
&D_{KL} [ q(\zhat \mid \x, \yhat=\y_1) \mid \mid q(\zhat \mid \x, \yhat=\y_0)] \\
&=D_{KL} [ N(\mu(\y_1), \Sigma(\y_1)) \mid \mid N(\mu(\y_0), \Sigma(\y_0))] \notag \\
&= \frac 1 2 \log\frac{\mathrm{det}(\Sigma(\y_0))}{\mathrm{det}(\Sigma(\y_1))} - \frac{m}{2} + \frac 1 2 \mathrm{tr} \left( \Sigma(\y_0)^{-1} \Sigma(\y_1) \right) \notag \\
&\quad + \frac 1 2 \mathrm{tr} \left\{ \Sigma(\y_0)^{-1} (\mu(\y_0) - \mu(\y_1)) (\mu(\y_0) - \mu(\y_1))^T  \right\} \,.
\end{align}
Applying~\cref{prop:det}, the first term follows
\begin{align}
\frac 1 2 \log\frac{\mathrm{det}(\Sigma(\y_0))}{\mathrm{det}(\Sigma(\y_1))} \le \frac{Mm}{2e} \Delta(\y_0, \y_1) \,.
\end{align}
Following the approach used in~\cref{eq:first_term}, we can derive
\begin{align}
\mathrm{tr} \left\{ \Sigma(\y_0)^{-1}  \Sigma(\y_1) \right\} 
\le M m \Delta(\y_0, \y_1)  \,.
\end{align}
Also, the following result can be found in~\cref{prop:trace_terms}.
\begin{align}
\mathrm{tr} \left\{ \Sigma(\y_0)^{-1} (\mu(\y_0) - \mu(\y_1)) (\mu(\y_0) - \mu(\y_1))^T  \right\} \notag \\
\le \frac{m L^2}{\lambda} \Delta(\y_0, \y_1)^2 \,.
\end{align}
Now, we combine all the terms and derive: 
\begin{align}
&D_{KL} [ q(\zhat \mid \x, \yhat=\y_1) \mid \mid q(\zhat \mid \x, \yhat=\y_0)] \notag \\
&\le \frac{M m}{2e} \Delta(\y_0, \y_1) - \frac{m}{2} + {M m} \Delta(\y_0, \y_1) \\
&\quad + \frac{m L^2}{\lambda} \Delta(\y_0, \y_1)^2 \notag \\
&= Mm \frac{1+2e}{2e} \Delta(\y_0, \y_1) - \frac{m}{2} + \frac{m L^2}{\lambda} \Delta(\y_0, \y_1)^2 \notag \\
&\le \frac{3Mm}{2} \Delta(\y_0, \y_1) - \frac{m}{2} + \frac{m L^2}{\lambda} \Delta(\y_0, \y_1)^2 \notag \\ 
&= m \left(\frac{L}{\sqrt{\lambda}} \Delta(\y_0, \y_1) + \frac{\sqrt{M \lambda}}{4 L} \right)^2 - m \left( \frac{9 M^2 \lambda + 8 L^2}{16 L^2} \right) \notag \\
&=O \left( m \Delta(\y_0, \y_1)^2 \right) \,,
\end{align}
since $m$ the dimension of $\z$ is also a variable in our task.
We finish the proof here.
\end{proof}

\section{Code and Implementation Details} \label{sec:impl}
Our code is publicly available at \emph{github}\footnote{
\url{https://github.com/huangweipeng7/lsnpc}
}. 
We implemented the $\beta$-VAE framework which could better learn the disentangled representations of data.
Additionally, Roskams-Hieter et al.~\cite{roskams2023leveraging} demonstrate its strong connection to \emph{power posteriors} that are more robust given mis-specified priors~\cite{li2023improved,friel2014improving}.
That is, with this implementation, the properties of our theoretical framework remain while the learned models could be more robust.
Note that $\beta$ will be only associated with the distributions centering on the latent variables $\z$ and $\zhat$.
The value of $\beta$ is fixed to be 0.01 in our implementation for all experiments.
Except for the sensitivity analysis, we set $\nu = \nu_0 = 2.01$ for all the experiments, which conforms to the requirements of $\nu = \nu_0 > 2$ as discussed in Theorem 2.
The hyperparameter $\eta$, the weights for the mixture proposal distribution in the supervised setting, was set to $0.5$ for all experiments.
All the experiments were conducted on a machine equipped with an Nvidia 4090D 24GB GPU, an Intel i7-13700KF CPU, and 32GB of RAM.

The label encoder is a 4-layer multilayer perceptron (MLP) that uses GELU~\cite{hendrycks2016gaussian} as its activation function. 
To improve stability and performance, Batch Normalization with a momentum of 0.01 is applied between the connecting layers.
The size of the fully connected layers in this MLP was fixed as 64.
In addition, the output label embedding size was 128.
The solution of merging the label embeddings and data feature embeddings is concatenation.
For all experimental configurations, we employed the AdamW optimizer along with a cosine annealing learning rate scheduler, which had a fixed cycle of 10 epochs (even when the training epochs were fewer than 10 in certain settings).
Additionally, we applied a batch size of 32.
All of the above hyperparameter settings were consistent for both the unsupervised and semi-supervised learning.
However, the numbers of training epochs and learning rates differ under the unsupervised and semi-supervised learning.
For all semi-supervised settings, the training epochs were fixed to 5.
While for the unsupervised settings, the numbers of training epochs for \vocseven{} and \voctwelve{} were respectively 20 and 15.
Furthermore, the epoch numbers were 20 and 15 for \tomato{} and \coco{} respectively.
More hyperparameter settings are separately listed in the public github repository for each configuration. 
The hyperparameters were manually optimized with regard to the performance of the validation sets.
Tuning the learning rate and number of training epochs were usually the most effective factors to improve the performance of \method{} in our experiments.

We employed the Adam optimizer throughout the training for base models.
The learning rate was fixed to 1e-5 for \tomato{} and 5e-5 for the other datasets.
Similarly, the batch size was set to 32 for \tomato{} and 128 for the other datasets.
Next, the number of training epochs was 20 for \vocseven{} and \voctwelve{}, 40 for \tomato{}, and 30 for \coco. 
For HLC, the number of epoch to start correcting was set to 5. 
These hyperparameter settings were mainly derived from~\cite{xia2023holistic,gehlot2023tomato} with minor modifications based on a rough grid search.
Since these approaches serve as the base models for assessing whether our method can further yield solid improvements, we did not put extensive effort into optimizing the hyperparameters.
Finally, for all configurations including the base models and \method{}, the checkpoint that achieves the best validation \microf{} was selected, since a high \macrof{} could only be obtained given that a high \microf is achieved. 
}


\bibliographystyle{IEEEtran}
\bibliography{nlc}
%


%
%
%
%
%
%
%
%
 
\end{document}